\documentclass[twoside,11pt]{article}

\usepackage{blindtext}

\usepackage{jmlr2e}

\usepackage{graphicx} 
\usepackage{amsmath}

\usepackage{comment}
\usepackage{paralist}

\usepackage{amsmath}%
\usepackage{MnSymbol}%
\usepackage{wasysym}%

\usepackage{algpseudocode}
\usepackage{algorithm}

\usepackage{enumitem}

\newenvironment{myassumptions}{%
   \begin{description}[style=multiline, leftmargin = 18pt, align=left,font=\normalfont]%
}{%
   \end{description}%
}
 \makeatletter
  \def\nl#1#2{\begingroup
     \scalebox{0.85}[1]{\textbf{#2}}%
    \def\@currentlabel{\textnormal{\scalebox{0.85}[1]{\textbf{#2}}}}
     \phantomsection\label{#1}\endgroup
}
\makeatother

\newcommand{\nat}{\mathbb{N}}

\newcommand{\reals}{\mathbb{R}}
\newcommand{\del}{\partial}
\newcommand{\CL}{\mathcal{L}}

\newcommand{\CE}{\mathbb{E}}
\newcommand{\bo}{\boldsymbol{1}}
\newcommand{\om}{\omega}
\newcommand{\CD}{\mathcal{D}}

\newcommand{\CA}{\mathcal{A}}

\newcommand{\hn}{\hat{n}}
\newcommand{\hx}{\hat{x}}
\newcommand{\Q}{Q}
\newcommand{\tQ}{\tilde{Q}}

\newcommand{\KLV}{\text{KLV}}
\newcommand{\RMP}{\text{RMP}}

\newcommand{\ps}{p^*}

\newcommand{\bdim}{d_b}
\newcommand{\adim}{d_{\vvar}}
\newcommand{\xdim}{d_x}

\newcommand{\tdim}{d_{\theta}}

\newcommand{\CO}{\mathcal{O}}

\newcommand{\CN}{\mathcal{N}}

\newcommand{\PR}{\mathbb{P}}

\newcommand{\CX}{\mathcal{X}}

\newcommand{\Xd}{X_{\text{data}}}

\newcommand{\law}{\text{Law}}

\newcommand{\CLd}{\CL_{\text{data}}}

\newcommand{\lf}{f}

\newcommand{\infbd}{I}

\newcommand{\clfint}{C_{\lf}}

\newcommand{\ut}{\underline{t}}

\newcommand{\ui}{\underline{i}}

\newcommand{\us}{\underline{s}}

\newcommand{\lfgsup}{A}

\newcommand{\vfbd}{M'}
\newcommand{\pmlen}{L}

\newcommand{\tPhi}{\tilde{\Phi}}

\newcommand{\tlam}{\tilde{\lambda}}

\newcommand{\CP}{\mathcal{P}}
\newcommand{\Fn}{\mathcal{F}_n}

\newcommand{\cfac}{c}
\newcommand{\Xt}{X^{\theta}}
\newcommand{\A}{A}
\newcommand{\vvar}{v}

\newcommand{\cn}{q} 

\newcommand{\lip}{b}

\newcommand{\phisol}{\phi_z^{\theta}(t)}
\newcommand{\phisoln}{\phi_z^{\theta}}

\newtheorem{mtheorem}{Meta-Theorem}
\newtheorem{mcorollary}{Meta-Corollary}

\usepackage{xr}

\usepackage{times}
\usepackage{hyperref}
\usepackage{subcaption}

\usepackage{adjustbox}

\usepackage{booktabs}
\usepackage{siunitx}
\usepackage{pgfplotstable}

\usepackage{csvsimple}

\usepackage{lastpage}

\ShortHeadings{Efficient Neural SDE Training using Wiener-Space Cubature}{Snow and Krishnamurthy}
\firstpageno{1}

\begin{document}

\title{Efficient Neural SDE Training using Wiener-Space Cubature \thanks{This research was supported by NSF grants CCF-2312198 and CCF-2112457 and U. S. Army Research Office under grant W911NF-24-1-0083}}

\author{\name Luke Snow \email las474@cornell.edu \\
       \addr School of Electrical and Computer Engineering\\
       Cornell University\\
       Ithaca, NY, USA
       \AND
       \name Vikram Krishnamurthy \email vikramk@cornell.edu \\
       \addr School of Electrical and Computer Engineering\\
       Cornell University\\
       Ithaca, NY, USA}

\editor{}

\maketitle

\begin{abstract}%
 A neural stochastic differential equation (SDE) is an SDE with drift and diffusion terms parametrized by neural networks. The training procedure for neural SDEs consists of optimizing the SDE vector field (neural network) parameters to minimize the expected value of an objective functional on infinite-dimensional path-space. Existing training techniques focus on methods to efficiently compute path-wise gradients of the objective functional with respect to these parameters, then pair this with Monte Carlo simulation to estimate the gradient expectation. In this work we introduce a novel training technique which bypasses and improves upon this Monte Carlo simulation; we extend results in the theory of Wiener space cubature to approximate the expected objective functional value by a weighted sum of functional evaluations of \textit{deterministic ODE solutions}. Our main mathematical contribution enabling this approximation is an extension of cubature bounds to the setting of Lipschitz-nonlinear functionals acting on path-space. Our resulting constructive algorithm allows for more computationally efficient training along several lines. First, it circumvents Brownian motion simulation and enables the use of efficient parallel ODE solvers, thus decreasing the complexity of path-functional evaluation. Furthermore, and more surprisingly, we show that the \textit{number of paths} required to achieve a given (expected loss functional oracle value) approximation can be reduced in this deterministic cubature regime. Specifically, we show that under reasonable regularity assumptions we can observe a $\CO(n^{-1})$ convergence rate, where $n$ is the number of path evaluations; in contrast with the standard $\CO(n^{-1/2})$ rate of naive Monte Carlo or the $\CO(\log(n)/n)$ rate of quasi-Monte Carlo. 
\end{abstract}

\begin{keywords}
  neural stochastic differential equations, generative modeling, stochastic dynamical systems, Wiener-space cubature, rough path theory
\end{keywords}


\section{Introduction}
\label{sec:intro}


A neural stochastic differential equation (SDE) \citep{tzen2019neural,liu2019neural,kidger2021neural,issa2024non,tzen2019theoretical,li2020scalable}  is an SDE with drift and diffusion terms parametrized by neural networks. Neural SDEs are highly flexible \textit{continuous-time generative models} for time-series data, capable of approximating arbitrary dynamics.  They are useful for modeling continuous-time systems with noise, as  they learn  the underlying stochastic processes directly from data, using the function approximation capabilities of neural networks. Applications include generative time-series tasks such as in quantitative finance \citep{gierjatowicz2020robust,arribas2020sig,choudhary2024funvol}.

The training procedure for neural SDEs involves  tuning the vector field (neural network) parameters to minimize a functional distance measure between the distributional law of the neural SDE and an empirical data distribution. Existing works innovate along two main directions: the development  of novel functional distance measures, and construction of more efficient computation of gradients with respect to parameters. The former  includes approaches such  as adversarial GAN objectives \citep{kidger2021neural}, variational free energy minimization \citep{tzen2019neural}, finite-dimensional distribution matching \citep{zhang2024efficient}, and signature kernel scores \citep{issa2024non}. The latter  includes methods such as automatic differentiation (AD) \citep{tzen2019neural} and stochastic adjoint methods \citep{kidger2021efficient}. 

Existing training techniques assume  that path-wise parameter gradient computations of neural SDE sample paths can be combined with Monte Carlo simulation to obtain gradients of the \textit{distributional} performance, quantified by the  expected value of an appropriate loss functional. Inherently accepted is that this Monte Carlo simulation fundamentally produces an estimate with statistical error $\CO(n^{-1/2})$, where $n$ is the number of SDE sample path gradient evaluations.

In this work, we move beyond this Monte Carlo training regime to construct a more efficient training scheme. Specifically, we extend results from  the theory of Wiener space cubature to express the expected loss functional as a functional of a \textit{weighted sum of deterministic ordinary differential equation (ODE) solutions}. This reformulation allows us to express the optimization objective as a minimization problem with respect to this functional of ODE solutions. Optimization of this derived expression offers two important advantages:
\begin{compactenum}
    \item \textit{Exploiting efficient ODE methods}. We improve upon existing neural SDE path-wise gradient computation by leveraging \textit{efficient ODE gradient computation methods}. One may circumvent the computational requirement of simulating and storing Brownian motion paths, reducing both time and memory complexity. 
    \item \textit{Improve upon Monte Carlo complexity}. Under reasonable assumptions, we improve upon the $\CO(n^{-1/2})$ Monte Carlo simulation guarantee, achieving $\CO(n^{-1})$ complexity. This drastically reduces the number of paths which must be solved to achieve a given approximation bound.
\end{compactenum}

Observe that the two advantages of this approach are complimentary: the latter means we can reduce the number of paths required to achieve any given loss approximation bound. The former means, that since these paths are now solutions of \textit{ODEs}, rather than \textit{SDEs}, each path evaluation is also cheaper in both time and memory overhead.

\textit{Pre-Processing Necessity}: These computational savings are nontrivial, but as such they are not free. As we will discuss, they rely on our ability to construct structured deterministic cubature paths in our ambient space, which we solve the ODEs with respect to. We call this construction a necessary pre-processing step, which must be run once before training the model. However, we provide a quantitative complexity bound for the computation required in this pre-processing step, and show numerically that even in high dimensions it is negligible compared to the gains in efficiency made during training.

\textit{Contribution -- Cubature Bounds for Nonlinear Functionals}: Our key mathematical contribution is an extension of cubature on Wiener space \citep{lyons2004cubature} to the nonlinear path functional domain. Cubature on Wiener space extends the idea of deterministic quadrature for numerical integration, as represented by Tchakaloff's theorem \citep{bayer2006proof}, to solutions of diffusions driven by Brownian motion. This approach was introduced by Lyons and Victoir \citep{lyons2004cubature} for numerically solving partial differential equations through their stochastic Fokker-Planck representation, via diffusion-law approximations by deterministic ODE solutions through well chosen "cubature" paths. Such results \citep{lyons2004cubature, litterer2012high, crisan2019cubature} are grounded in the Lie-algebraic structures of rough path theory \citep{lyons1998differential}, and quantify approximation error at single-time instances of (functions of) the diffusion law. These approximation bounds can be extended straightforwardly to linear \textit{functionals} operating on path-space, by Riesz representation and uniform integral bounds. However, many useful neural SDE loss functionals are \textit{nonlinear} functionals \citep{li2020scalable,kidger2021neural,zhang2024efficient}; our main mathematical contribution is the development of Wiener-space cubature bounds for general nonlinear functionals, which need only satisfy a minimal Lipschitz condition. Other works have investigated such cubature approximation for infinite-dimensional nonlinear functionals, e.g., \cite{bayer2013cubature}, but only in the asymptotic weak convergence sense. As far as we are aware, we are the first to provide non-asymptotic approximation bounds which allow for complexity quantification, in this generalized domain. This should be of interest outside the domain of neural SDE training.

Next we introduce the neural SDE model in Section~\ref{sec:nsde}, discuss the learning framework in Section~\ref{sec:learnfram}, and outline our contributions in more detail in Section~\ref{sec:contrib}. 


\subsection{Neural SDE}
\label{sec:nsde}

In this section we introduce the neural SDE as an It\^o stochastic differential equation since that is the prevailing form in the neural SDE literature. However, in Section~\ref{sec:nsdecub} we reformulate it as a Stratonovich SDE for analytical convenience and consistency with the cubature methodology.

Fix a finite time horizon $T>0$, and let $B: [0,T] \to \reals^{\bdim}$ be a $\bdim$-dimensional standard Brownian motion. Let $\vvar \sim \CN(0,I_{\adim})$ be a $\adim$-dimensional normal random variable.
We define the neural SDE model analogously to \cite{kidger2021neural,issa2024non,zhang2024efficient}
, as
\begin{align}
\begin{split}
\label{eq:nsde}
 \Xt_0 = \zeta^{\theta}(\vvar), \quad dX_t^{\theta} = \mu^{\theta}(t,X_t^{\theta})dt + \sigma^{\theta}(t,X_t^{\theta})dB_t
\end{split}
\end{align}
with $t\in[0,T]$, and
\begin{align}
    \zeta^{\theta} :\reals^{\adim} \to \reals^{\xdim}, \quad \mu^{\theta}: [0,T] \times \reals^{\xdim}\to \reals^{\xdim}, \quad \sigma^{\theta}: [0,T] \times \reals^{\xdim} \to \reals^{\xdim\times \bdim}
\end{align} 
Here $\zeta^{\theta}, \mu^{\theta}, \sigma^{\theta}$ are \textit{neural networks} collectively parametrized by weights $\theta \in \Theta$, where $\Theta$ is a compact subset of $\reals^{\tdim}$. In practice, the neural network architecture automatically confines the parameter space to this compact domain $\Theta$. We specify assumptions on the structure of the neural networks in Section~\ref{sec:prelim}, such that \eqref{eq:nsde} has a unique strong solution for all $t\in[0,T]$, denoted by $X_t^{\theta}\in\reals^{\xdim}$.

\subsection{Learning Framework}
\label{sec:learnfram}
In this section we outline the framework for training neural SDEs, and provide an outline of our main results. Neural SDEs are trained such that their path-space distributional statistics approximate some empirical (training) data in path-space.

\vspace{0.2cm}
 \textit{Neural SDE Path-Space Distribution}: Let $\CX$ be the space of continuous paths on $[0,T]$, and denote $\Xt \in \CX$ a sample path of the neural SDE \eqref{eq:nsde} with parameter $\theta$. We index path $X^{\theta}$ in time as $(X^{\theta}_t)_{t\in[0,T]}$. Let $\CP(\theta) = \law((X_t^{\theta})_{t\in[0,T]})$ be the distributional law of the process $\Xt$. 

\vspace{0.2cm}
\textit{Empirical Training Data}: We assume access to an empirical data distribution $\Xd$ on path-space $\CX$, i.e., a collection of paths on $\CX$. We aim to train \eqref{eq:nsde} such that $\CP(\theta) \approx\Xd$; such that $\CP(\theta)$ approximates $\Xd$ in a distributional sense. This approximation can be quantified rigorously by utilization of a loss functional, as we now discuss.

\subsubsection{Loss Functional}
\label{sec:lossfun}

The procedure for training a neural SDE exploits the construction of a loss functional $\CL : \CX \times \CX \to \reals$, which compares two paths on $\CX$. Then, $\CP(\theta) \approx \Xd$ can be identified with minimizing $\CE_{\Xt\sim \CP(\theta), X' \sim \Xd}[\CL(X,X')]$. 
In this work we assume a fixed empirical data distribution $\Xd$, without loss of generality. Then for ease of notation we represent the training objective through a \textit{data-dependent} loss functional \[\CLd = \CE_{X\sim \Xd}[\CL(\cdot,X)]: \CX \to \reals\] which quantifies the \textit{distance} between neural SDE paths $\Xt\in \CX$ and the specified data distribution $\Xd$. In terms of the learning framework, we may then identify the relation $\CP(\theta) \approx \Xd$ as occurring when $\CE[\CLd(\Xt)]$ is small. Thus, the training objective for a neural SDE, given empirical data distribution $\Xd$, is to solve the optimization problem\footnote{As is standard in this framework, we only aim to compute a local stationary point, since global minimization is in general intractable. This does not detract from our work, which simply provides efficiency gains for existing state-of-the-art methods for neural SDE optimization.}:
\begin{equation}
\label{eq:objective} \arg\min_{\theta\in\Theta \subset\reals^{\tdim}}\CE_{\Xt \sim \CP(\theta)}[\CLd(\Xt)]
\end{equation}
\cite{kidger2021neural,issa2024non,zhang2024efficient,briol2019statistical,gierjatowicz2022robust,cuchiero2020generative,li2020scalable} explore the performance of neural SDE models under different constructions of (what is equivalently) $\CLd$. In contrast, in this work we investigate an efficient method for estimating the general objective $\CE_{Xt \sim \CP(\theta)}[\CLd(\Xt)]$, and thus for training via backpropagation of gradients of this estimate. 

\textit{Existing Training Methodology}: State-of-the-art techniques to optimizing \eqref{eq:objective} operate along three main principles. First, path-wise simulation and gradient evaluation $\frac{\del}{\del\theta}\CLd(\Xt)$. Then, Monte Carlo simulation to approximate $\CE\left[\frac{\del}{\del \theta}\CLd(\Xt)\right] = \frac{\del}{\del \theta}\CE[\CLd(\Xt)]$ with $\CO(n^{-1/2})$ error. Third, optimization using gradient methods. Most innovation in training occurs in the first step, with the construction of loss functionals and efficient methods for SDE gradient evaluation. In contrast, in this work we develop a methodology for bypassing the Monte Carlo simulation by an efficient deterministic approximation scheme, namely Wiener-space cubature.

\subsubsection{Main Results: Efficient Learning using Wiener Space Cubature}
\label{sec:contrib}

In this paper we present an alternative approach to gradient backpropagation which circumvents the stochastic (Monte Carlo) simulation paradigm. We extend Wiener space cubature techniques \citep{lyons2004cubature} to approximate the expected loss functional $\CE[\CLd(\Xt)]$ by a weighted sum of functionals of \textit{deterministic ODE solutions}. 
Specifically, our main result can be stated informally as follows:

\begin{mtheorem}[Wiener Space Cubature Approximation]
\label{thm:meta}
The expected loss functional in \eqref{eq:objective} can be approximated by a convex combination of functionals of ordinary differential equation trajectories:
\begin{align}
\begin{split}
\label{eq:mto}
\CE[\CLd(\Xt)] &= \sum_{z=1}^n \lambda_z \CLd(\phisol)+ \text{error}
\end{split}
\end{align}
where $\phisol$ is the solution at time $t$ of an ODE 
w.r.t. vector fields in \eqref{eq:nsde} and "cubature" path $\om_z(t), z= 1,\dots,n$. Here $n$ is the number of ODEs to be solved and the weights satisfy $\lambda_z \geq 0, \,\,\sum_{z=1}^n\lambda_z = 1$.  Also, under reasonable assumptions we achieve the following approximation bound:
\[\text{error} = \CO\left(n^{-1}\right)\]
\end{mtheorem}

Observe that Meta-Theorem 1 implies the following 
\begin{equation}
\label{eq:mineq}
\min_{\theta\in\Theta} \sum_{z=1}^n \lambda_z\CLd(\phisol) = \min_{\theta \in\Theta}\CE[\CLd(\Xt)] + \CO\left(n^{-1}\right)
\end{equation}
Thus, instead of optimizing the objective \eqref{eq:objective}, we  optimize the left hand side of \eqref{eq:mineq}. This allows us to not only achieve $\CO(n^{-1})$ approximation complexity, but to compute backpropagated gradients more efficiently using ODE methods, as captured in the following Meta-Corollary.

\begin{mcorollary}[Efficient ODE Gradient Evaluations]
\label{cor:mcor}
Given the representation \eqref{eq:mto}, we can optimize the approximate loss functional \eqref{eq:mineq} by computing gradients as
\begin{equation}
\label{eq:gradeval}
    \frac{\del}{\del\theta}\sum_{z=1}^n \lambda_z \CLd(\phisol) = \sum_{z=1}^n\lambda_z \frac{\del}{\del\theta} \CLd(\phisol)
\end{equation}
where $\frac{\del}{\del\theta} \CLd(\phisol)$ can be computed using automatic differentiation or ODE adjoint gradient methods \citep{chen2018neural,kidger2021hey,matsubara2021symplectic,mccallum2024efficient}. These techniques tend to be both more time- and memory- efficient than analogous methods for SDE gradient computation, such as automatic differentiation through SDE solvers \citep{tzen2019neural} or stochastic adjoint methods \citep{kidger2021efficient}.
\end{mcorollary}

Meta-Theorem~\ref{thm:meta} and Meta-Corollary~\ref{cor:mcor} represent the two main computational advantages of this approach, specifically: improvement upon Monte Carlo functional estimation complexity, and computation enhancement by removal of the requirement of simulating and storing Brownian motion, since we solve deterministic ODEs. Let us finally make a few remarks on the construction of this method.

\textbf{Remarks}. \begin{compactenum}

\item{\em Cubature Path and Weight Construction.} The cubature paths $\om_z$ and weights $\lambda_z$, $z\in[n]$, will be constructed such they are \textit{independent of $\theta$}. Thus, we need only compute them once prior to gradient descent, in a pre-processing operation, as discussed in Section~\ref{sec:hor}.

\item{ \em Sparse Path and Weight Recombination}. The key operation that allows our $\CO(n^{-1})$ complexity bound is a high-order recombination procedure introduced in \cite{litterer2012high}. This procedure  \textit{sparsifies} the cubature weights thereby drastically reducing the number of ODEs to be solved while maintaining desirable accuracy. This is made rigorous in Section~\ref{sec:hor}. 
\end{compactenum}

\textit{Organization}: Section~\ref{sec:nsdecub} outlines the Wiener space cubature methodology and provides our first result, Theorem~\ref{thm:cubap}, which extends the cubature techniques to approximate \textit{functionals} on infinite-dimensional path-space. Section~\ref{sec:hor} introduces a high-order recombination procedure which drastically reduces the number of necessary ODE computations to achieve a desirable cubature approximation. It contains our second main result, Theorem~\ref{thm:bd2}, which provides the precise guarantees that were presented in Meta-Theorem~\ref{thm:meta}. Section~\ref{sec:numerics} provides a detailed high-dimensional numerical implementations of our method, verifying the improvements in gradient estimation complexity and training efficiency. All the results in this paper are fully reproducible and the  code is available  at \texttt{\small https://github.com/LukeSnow0/Neural-SDE}. \textit{All proofs are  in the appendix}. 

\section{Neural SDE Wiener Space Cubature}
\label{sec:nsdecub}
In this section we first introduce the existing methodology of Wiener space cubature approximations. We then extend these techniques to handle loss \textit{functionals}, enabling application to neural SDE training. The main result of this section is Theorem~\ref{thm:cubap}, which provides explicit error bounds between the expected loss functional and our constructed Wiener space cubature approximation.


\subsection{Wiener Space Cubature Formulation}
\label{sec:cubform}

In this section we first introduce the \\ Stratonovich SDE reformulation of the 
 neural SDE \eqref{eq:nsde}. We then introduce the fixed-time Wiener space cubature formulation of \cite{lyons2004cubature}, which utilizes this Stratonovich form. 

\subsubsection{Stratonovich Reformulation}
For notational convenience and consistency with the cubature literature, we absorb the time-dependency in \eqref{eq:nsde} into the state space. This allows us to express vector fields in terms of only the augmented state vector $\hat{X}_t$. First, recall the neural SDE \eqref{eq:nsde} in It\^o form.

The time-dependence of vector fields $\mu^{\theta}(t,\cdot), \sigma^{\theta}(t,\cdot)$ can be absorbed into the state-space by constructing the following equivalent augmented system:  Let $\hat{X}_{t,[i:j]}$ denote the $i$th to $j$th elements of augmented vector $\hat{X}_t$, and $\hat{X}_{t,[i]}$ denote the $i$'th element. For variable $\hat{X}_t\in \reals^{\xdim+1}$, with $\hat{X}_{t,[0]} = t$, let
\begin{align}
\begin{split}
\label{eq:tindsde}
    \hat{X}_{0,[1:\xdim]} &= \zeta^{\theta}(\vvar), \quad d\hat{X}_t = \hat{\mu}^{\theta}(\hat{X}_t)dt + \hat{\sigma}^{\theta}(\hat{X}_t)dB_t, \quad
    \hat{X}_{0,[0]} = 0,\quad  d\hat{X}_{t,[0]} = dt \\
    &\zeta^{\theta} :\reals^{\adim} \to \reals^{\xdim}, \quad \hat{\mu}^{\theta}: \reals^{\xdim+1}\to \reals^{\xdim+1}, \quad \hat{\sigma}^{\theta}: \reals^{\xdim+1} \to \reals^{\xdim+1\times \bdim}
\end{split}
\end{align}

Next, let $\{V_i^{\theta}\}_{i=0}^{\bdim}$ be the vector fields in the Stratonovich formulation of the augmented neural SDE \eqref{eq:nsde}, with $X_t^{\theta} \in \reals^{\xdim+1}$. Specifically,
\[\,V_0^{\theta}(x) = \hat{\mu}^{\theta}(x) - \frac{1}{2}\sum_{i=1}^{\bdim} dV_i^{\theta}(x)\cdot V_i^{\theta}(x),\, \quad \,V_i^{\theta}(x) = \hat{\sigma}_i^{\theta}(x), \,\,i=1,\dots,\bdim\]
Then, 
\begin{align}
\begin{split}
\label{eq:strat}
    X_t^{\theta} &= \int_0^t V_0^{\theta}(X_s^{\theta})ds + \int_0^t\sum_{i=1}^{\bdim} V_i^{\theta}(X_s^{\theta}) \circ dB_s^i = \int_0^t\sum_{i=0}^{\bdim} V_i^{\theta}(X_s^{\theta}) \circ dB_s^i ,\quad 
    X_{t,[0]}^{\theta} = t
\end{split}
\end{align}
with $\circ\,dB_s^i$ denoting Stratonovich stochastic integration, and  by convention $\circ\, dB_t^0 = dt$. 
The initial value in the  Stratonovich formulation is chosen identical to  \eqref{eq:nsde}; therefore, the process described by 
 \eqref{eq:strat} is equivalent to the original process \eqref{eq:nsde}.

\subsubsection{Wiener Space Cubature}
Given the Stratonovich SDE formulation \eqref{eq:strat}, the following is the main Wiener-space cubature result of \cite{lyons2004cubature}; this result pioneered the usage of deterministic cubature approximations in stochastic analysis \citep{crisan2019cubature} and finance \citep{teichmann2006calculating}, for instance. 

\begin{theorem}[\cite{lyons2004cubature}]
\label{thm:lyons}
Denote by $C_{bv}^0([0,T],\reals^{d})$ the space of $\reals^{d}$-valued continuous functions of bounded variation defined on $[0,T]$. Let $(X_t^{\theta})_{t\in[0,T]}$ be the solution of a SDE in Stratonovich form \eqref{eq:strat}, with vector fields $\{V_i^{\theta}\}_{i=0}^{\bdim}$. There exist, for any sufficiently smooth $f: \reals^{\xdim+1} \to \reals$, paths and weights 
\begin{equation}
\label{eq:pathwt}
\om_1,\dots,\om_{\cn} \in C_{bv}^0([0,T],\reals^{\bdim+1}), \quad \lambda_1,\dots,\lambda_{\cn} > 0: \,\sum_{i=1}^{\cn}\lambda_i = 1
\end{equation} such that for $t > 0$ , 
\begin{equation}
\label{eq:w_cube}
\CE[f(X_t^{\theta})] = \sum_{j=1}^{\cn}\lambda_j f(\phi_j^{\theta}(t)) + \CO(t^{\frac{m+1}{2}})
\end{equation}
Here $m$ is the cubature "order", and $\phi_j^{\theta}(t)$ is the solution at time $t$ of the following \textit{ordinary} differential equation with respect to $\om_j$, in the sense of \cite{lyons2004cubature}:
\begin{align}
\begin{split}
\label{eq:ODE}
&d\phi_j^{\theta}(t)= \sum_{i=0}^{\bdim}
V_i^{\theta}(\phi_j^{\theta}(t))d\om_j^i(t), \quad \phi_0^{\theta} = X_0 = \zeta^{\theta}(\vvar)
\end{split}
\end{align}

\end{theorem}

\begin{proof}
See Appendix
\end{proof}

Theorem~\ref{thm:lyons} allows one to approximate the solutions of an SDE by computing solutions of ODEs with respect to the SDE-defining vector fields $\{V_i^{\theta}\}_{i=1}^{\bdim}$ and well-chosen "cubature paths" $\{\om_i\}_{j=1}^q$ and weights $\{\lambda_j\}_{j=1}^{q}$. The cubature "order", defining the structure of these paths and weights, essentially controls the precision of this approximation \eqref{eq:w_cube} and can be controlled such that this method is more efficient than Monte Carlo baselines for computing $\CE[f(X_t^{\theta})]$. In this work we extend Theorem~\ref{thm:lyons} to general Lipschitz-nonlinear functionals of the entire SDE path $\{X_t^{\theta}, t\in[0,T]\}$, and present constructive algorithms for utilizing this approximation for neural SDE training. Next we provide some insight into the construction of the cubature paths and weights defining \eqref{eq:w_cube}, but refer to \cite{lyons2004cubature} for the full details.  

\subsubsection{Cubature Path and Weight Construction}
\label{sec:cpwc}
\cite{lyons2004cubature} provides not only the existence of cubature paths and weights defining an approximation of the form \eqref{eq:w_cube}, but provides quantitative conditions on these paths and weights that can lead to their construction \eqref{eq:w_cube}.
Furthermore, explicit paths and weights \eqref{eq:pathwt} of certain orders $(m=3,5)$ are constructed which satisfy \eqref{eq:w_cube}. These paths and weights can be abstractly considered as approximating the statistics of the underlying Brownian motion driving the system \eqref{eq:nsde}; here we provide their formal defining property. This definition relies on a stochastic Taylor expansion and is based in rough path theory, see \cite{lyons2004cubature} for full details.

\begin{definition}[Cubature Path and Weights]
\label{def:cpw}
Denote
\begin{equation}
\label{eq:CAm}
\A = \bigcup_{k=0}^{\infty} \{0,\dots,\bdim\}^k,\,\, \CA_m = \{(i_1,\dots,i_k) \in \A\backslash\{\emptyset,0\}: k + \text{card}\{j: i_j=0\} \leq m\}
\end{equation}
where $\text{card}\{S\}$ returns the cardinality of the set $S$. The paths $\om_1,\dots,\om_{\cn} \in C_{bv}^0([0,T],\reals^{\bdim})$ and weights $\lambda_1,\dots,\lambda_{\cn} > 0$
define a cubature formula on Wiener space of degree $m$ at time $T$ if for all $(i_1,\dots,i_k) \in \CA_m$,
\begin{equation}
\label{eq:cubdef}
\CE\left(\int_{0 < t_1 < \dots < t_k = T}\circ dB_{t_1}^{i_1}\dots\circ dB_{t_k}^{i_k} \right) = \sum_{j=1}^{\cn} \lambda_j \int_{0 < t_1 < \dots < t_k = T} d\om_j^{i_1}(t_1)\dots d\om_j^{i_k}(t_k)
\end{equation}
\end{definition}

The paths and weights necessary to achieve the approximation \eqref{eq:w_cube} are precisely those that define a cubature formula on Wiener space of degree $m$ at time $t$. These will be used in our cubature path construction in Section~\ref{sec:cpc}.

The following Lemma provides the \textit{existence} of such paths and weights, enabling Theorem~\ref{thm:lyons}. This Lemma can be viewed as an extension of Tchakaloff's Theorem \citep{bayer2006proof}. 
\begin{lemma}[\cite{lyons2004cubature} Theorem 2.4]
\label{lem:lv4}
    Let $m$ be a natural number. Then there exist $\cn$ paths $\om_1,\dots,\om_{\cn} \in C_{bv}^0([0,T],\reals^{\bdim+1})$ and $\cn$ positive weights $\lambda_1,\dots,\lambda_{\cn}$ with $\cn \leq \text{card}\{\CA_m\}$, defining a cubature formula on Wiener space of degree $m$ at time $T$.
\end{lemma}

The contribution of \cite{lyons2004cubature} is to provide not only existence of such cubature paths and weights, but explicit constructions of them for certain orders $m$. Definition~\ref{def:cpw} provides the structure necessary to explicitly specify such paths and weights, and \cite{lyons2004cubature} does so using Lie-algebraic synthesis rooted in this stochastic Taylor expansion, for orders $m=3,5$. An explicit order $m=7$ construction is given in \cite{ferrucci2024high} using Hopf algebraic techniques. \textit{For our purposes, we omit these construction techniques and simply rely on the fact that certain explicit constructions exist. We can directly use these existing cubature paths and weights within our constructive algorithms for efficient neural SDE training}. First, we must attain the noted generalization of Theorem~\ref{thm:lyons} to nonlinear functionals.

\subsection{Wiener Space Cubature for Nonlinear Loss Functional}

In Section~\ref{sec:cubform} we introduced the structure of cubature formulations on Wiener space, which approximate the expected value of functions of solutions of SDEs at \textit{single time-instances}. Recall that the training procedure for neural SDEs involved minimization of a data-dependent loss \textit{functional}, which acts on entire paths rather than solutions at a single time instance. In this section we introduce an algorithmic procedure for extending the approximation capabilities of cubature methods to the nonlinear path functional domain. Put simply, we extend Wiener space cubature to a function space. 

\subsubsection{Preliminaries: Assumptions and Structural Definitions}
\label{sec:prelim} 
Recall the notation of Section~\ref{sec:intro}. We denote by $(X_t^{\theta})_{t\in[0,T]}$ the unique strong solution of the neural SDE \eqref{eq:nsde}. We are interested in forming cubature approximations of the expected path loss $\CE[\CLd(\Xt)]$ with respect to loss functional $\CLd$, in the sense of a generalization of Theorem~\ref{thm:lyons}.
We now introduce several assumptions  which will be necessary for our generalized cubature approximation bounds, introduced in Section~\ref{sec:cpc}, to hold.

\begin{myassumptions}
\item[\nl{as:vecbd}{A1}]
The vector fields $\zeta^{\theta}, \mu^{\theta}, \sigma^{\theta}$ in the It\^o neural SDE formulation \eqref{eq:nsde} are Lipschitz-continuous, and $\CE_{\vvar}[\zeta^{\theta}(\vvar)^2] < \infty$. Furthermore, assume there exists a constant $M < \infty$ such that  
\[\sup_{t\in[0,T],x\in\reals^{\xdim},\theta\in\Theta}\left\{\|\mu^{\theta}(t,x)\| + \|\frac{\del}{\del x}\mu^{\theta}(t,x)\| + \|\sigma^{\theta}(t,x)\| + \|\frac{\del}{\del x}\sigma^{\theta}(t,x)\|\right\} \leq M\]
i.e., the  vector fields are uniformly bounded in magnitude and in first-order spatial differentiation. 

\item[\nl{as:intfun}{A2}] The loss functional $\CLd(X): \CX \to \reals$ is $\lip$-Lipschitz continuous with respect to the $L^1$ metric over $\CX$; that is, 
\[|\CLd(X)-\CLd(Y)| \leq \lip\|X-Y\|_1 = \lip\int_0^T|X(t) - Y(t)|dt, \quad \forall X,Y\in \CX\]

\item[\nl{as:uh}{A3}] 
 Recall $\CA_m$ in \eqref{eq:CAm}, and inductively define a family of Lie-bracketed vector fields as 
 \[V_{[\emptyset]}^{\theta} = 0, \quad V_{[i]}^{\theta} = V_i^{\theta}, \quad V_{[\alpha*i]}^{\theta} = [V_{[\alpha]}^{\theta},V_i^{\theta}], \quad 0\leq i\leq \bdim, \,\,\,\alpha\in\A \]
where $[V_i,V_j] = J_{V_j}V_i - J_{V_i}V_j$ is the Lie bracket expansion and $J_V$ is the Jacobian matrix of vector field $V$.
Then, we assume the vector fields $\{V_i^{\theta}\}_{i\in[\bdim+1]}$ satisfy the uniform H\"ormander's condition: there exists an integer $\ps$ such that 
\begin{equation}
\inf\left\{\sum_{\alpha\in A_1(\ps)}\langle V_{[\alpha]}^{\theta}(x),\zeta \rangle^2; \, x,\zeta \in \reals^{\xdim}, |\zeta| = 1, \, \theta\in\Theta \right\} := M > 0
\end{equation}

\end{myassumptions}

\textit{Discussion of Assumptions}: The continuity and growth assumptions in \ref{as:vecbd} are the standard conditions guaranteeing unique strong solutions to \eqref{eq:nsde} \citep{kloeden1992stochastic}. The uniform boundedness in \ref{as:vecbd} is assumed in \cite{litterer2012high}, and can easily be satisfied in practice; with arbitrarily high probability the process $X_t^{\theta}$ will be confined to some compact subset of the domain $\reals^{\xdim}$ on the finite time interval $[0,T]$. Thus, in computation we effectively consider the behavior of vector fields $\mu^{\theta}$ and $\sigma^{\theta}$ within compact sets, in which case the boundedness immediately follows. \ref{as:intfun} admits many standard neural SDE loss functionals, for instance the latent SDE formulation of \cite{li2020scalable}. \ref{as:uh} is known as the uniform H\"ormander condition, and was introduced in \cite{kusuoka1987applications}. It is widely assumed in formal analyses of SDEs as it ensures the smoothness of the solution of SDEs; see Section 5 of \cite{elst2009uniform} for a detailed treatment. Furthermore, in Section~\ref{sec:UHcond} of the supplementary document we provide sufficient conditions on practical neural SDE vector fields for \ref{as:uh} to be met; these conditions are easily satisfied e.g., when we have uniform ellipticity of the diffusion matrix\footnote{Specifically, a sufficient condition for \ref{as:uh} to be met is if the diffusion matrix $g_\theta(t,x)=[V_1^\theta\;\cdots\;V_{\bdim}^\theta]\in\mathbb{R}^{d\times m}$ has full rank $d$ and is uniformly nondegenerate: $
g_\theta(t,x)\,g_\theta(t,x)^\top \;\succeq\; \lambda I_d, \forall (t,x,\theta),$ for some $\lambda>0$. In this case $\ps=1$ suffices.} This is .

\subsubsection{Cubature Path Construction and Error}
\label{sec:cpc}

The purpose of this subsection is twofold: we first provide an explicit algorithmic construction of ODEs with respect to paths in $\CX$, exploiting explicit constructions of cubature paths and weights defined by Definition~\ref{def:cpw}. Second, we analyze this construction to show that we can approximate the loss functional statistic $\CE[\CLd(\Xt)]$ arbitrarily well by the solutions of these ODEs. Theorem~\ref{thm:cubap} is our first main result, and provides a guarantee on this approximation. This construction, using interval-appendages of time-scaled paths, is exactly analogous to how \cite{lyons2004cubature} obtains desirable approximation rates using cubature formulae.

Our procedure for approximating a loss functional satisfying \ref{as:intfun} by deterministic cubature ODEs is as follows:
\begin{compactenum}
\item For $\cn,m > 0$, take cubature paths $\{\om_j\}_{j\in[\cn]}$ and weights $\{\lambda_j\}_{j\in[\cn]}$ which define a cubature formula on Wiener space of degree $m$ at time $1$, as in Definition~\ref{def:cpw}. Explicit constructions for such paths and weights are given in \cite{lyons2004cubature} for orders $m=3$ and $m=5$ and in \cite{ferrucci2024high} for order $m=7$.

\item Divide the interval $[0,T]$ as 
\begin{equation}
\label{eq:discint}
0 = t_0 < t_1 < \dots < t_k = T,
\end{equation}
and let $s_i = t_i - t_{i-1}, \, i= 1,\dots, k$. 

\item Construct the exhaustive set of vectors $\{I_z \in \nat^k\}_{z\in[\cn^k]}$ with each element taking a natural number $1,\dots,\cn$ and index beginning at 0. 
Construct scaled paths $\{\om_{s,j}\}_{j\in[\cn]}$ as 
\begin{equation}
\om_{s,j}: [0,T] \to \reals^{\bdim+1},\,\, \om_{s,j}^0(t) = t,\,\, \om_{s,j}^i(t) = \sqrt{s}\om_{j}^i(t/s), \text{ for } i=1,\dots, \bdim
\end{equation}
Then, construct paths $\om_z: [0,T] \to \reals^{\bdim+1}, z\in[\cn^k]$ as 
\begin{equation}
\label{eq:pathcon}
    \om_z(t) = \sum_{j=1}^{\cn} \sum_{i=0}^k\bo(t\in[t_{i},t_{i+1}))\bo(I_z[i] = j)\om_{s_{i+1},j}(t)
\end{equation}

\item Let $\phisol$ be the solution of the ODE 
\begin{equation}
\label{eq:ode}
d\phisol = \sum_{i=0}^{\bdim} V_i^{\theta}(\phisol)d\om_z^i(t),\quad \phi_0^{\theta} = X_0 = \zeta^{\theta}(\vvar)
\end{equation}
at time $t$, and where the initial point $\zeta^{\theta}(\vvar)$ is specified in \eqref{eq:nsde}. Denote by $\phisoln$ the entire path $\{\phisol, t\in[0,T]\} \in \CX$

\end{compactenum}

Then, letting $\lambda_z' := \lambda_{I_z[0]}\dots\lambda_{I_z[k]}$, the functional
\begin{align}
\label{eq:sum}
    \Phi_{\CLd}^{\theta} := \sum_{z\in [\cn^k]}\lambda_z' \CLd(\phisoln)
\end{align}
forms an approximation of the loss functional statistic $\CE[\CLd(\Xt)]$, with approximation error given by the following Theorem. This is the first main result of the paper, and constitutes our most significant mathematical contribution. It extends the Wiener-space cubature guarantees of \cite{lyons2004cubature} to general Lipschitz-nonlinear path functionals, and thus should be of interest outside the current setting of neural SDE training. 

\begin{theorem}[Cubature Path Approximation Error] 
\label{thm:cubap}
Recall the notation in Section~\ref{sec:cpc}. Let $t_i = T\left(1 - \left(1-\frac{i}{k}\right)^{\gamma}\right)$ and $0 < \gamma < m-1$. Recall $m$ is the cubature order in Definition~\ref{def:cpw}. Under \ref{as:vecbd}, \ref{as:intfun} and \ref{as:uh} we have: 
\begin{align}
\begin{split}
\label{eq:cpae}
&\left|\CE[\CLd(\Xt)] - \Phi_{\CLd}^{\theta}\right|\leq 3\,\lip \,T^{1/2}\,K \, C(m,\gamma) k^{-\gamma/2}
\end{split}
\end{align}
where $K$ is a constant independent of $T$, $C(m,\gamma)$ is a constant depending only on $m$ and $\gamma$.
\end{theorem}
\begin{proof}
See Appendix
\end{proof}
Thus, under suitable  time discretization, we can  approximate the loss functional statistic \\$\CE[\CLd(\Xt)]$ \textit{arbitrarily well}, taking $k$ large enough, by solving the \textit{deterministic} system of ODEs \eqref{eq:sum}. \footnote{Recall that we require the inequality
$0 < \gamma < m-1$
to hold. In practice, we need an explicit construction of cubature paths and weights of order $m$, as in Definition~\ref{def:cpw}. Thus, currently we are limited to several specific choices for $m$ (for instance, \cite{lyons2004cubature} provide explicit constructions for $m=3,5$). However, there is active research in providing explicit constructions of Wiener space cubature paths and weights of higher order; for instance \cite{ferrucci2024high} constructs order $m=7$ Wiener space cubature paths using Hopf algebras. \textit{Such advancements to higher-order $m$ cubature approximations will allow our bound \eqref{eq:cpae} to become tighter by allowing $\gamma$ to increase.} However, the current constructions for $m=3$ and $5$ given in \cite{lyons2004cubature} are sufficient for the computational gains detailed in this paper.} This result extends the guarantees of \cite{lyons2004cubature} to the nonlinear path functional domain, and may be of independent interest outside of the neural SDE training regime. The proof utilizes the Stone-Weierstrass Theorem, and can be found in Section~\ref{sec:pf3}.

\textit{Remark. 1-Wasserstein metric structure.} The bound \eqref{eq:cpae} is derived from a bound on the $1$-Wasserstein distance between the neural SDE path-space law and the discrete cubature-induced path-space law. Indeed, by duality this $1$-Wasserstein distance is equivalent to the supremum of the difference in expected $1$-Lipschitz functional evaluations, with respect to each law. Thus, such a bound implies \eqref{eq:cpae} for all nonlinear functionals $\CLd$ satisfying the Lipschitz structure \ref{as:intfun}.

\textit{Exponential Complexity}: The price of achieving the approximation error \eqref{eq:cpae} is that the number of ODEs \eqref{eq:sum} to be solved grows \textit{exponentially} ($\cn^k$) with the discretization \eqref{eq:discint} resolution $k$. In the following section we employ the high-order recombination method of \cite{litterer2012high} to \textit{drastically reduce the number of ODEs to be solved}, while \textit{maintaining desirable approximation error}. We do so in a constructive way that can result in an overall improvement in computational efficiency as compared to classical Monte Carlo stochastic simulation.

\section{High-Order Recombination for Efficient Neural SDE Cubature Approximations}
\label{sec:hor}

This section establishes the second main result of the paper, namely Algorithm~\ref{alg:wscr} and Theorem~\ref{thm:bd2}. 
In particular, we show how to exploit the high-order recombination technique of \cite{litterer2012high} in order to drastically reduce the number of ODE evaluations required for a given approximation error. The idea is to partition the space of paths and re-weight within these partitions, such that the center of mass is preserved and the number of paths in the partition is reduced. First we introduce some preliminary definitions which will be used in the main recombination algorithm. Then we provide our second main result, Theorem~\ref{thm:bd2}, which provides a quantitative approximation bound resulting from this recombination procedure. The bound in Theorem~\ref{thm:bd2} can be reduced, for certain parametrizations, such that an $\CO(n^{-1})$ rate is achieved; this is detailed in Corollary~\ref{cor:noneap}. Furthermore, we provide explicit complexity bounds for the computation required to perform the pre-processing operation (recombined cubature path construction - Algorithm~\ref{alg:wscr}) necessary to achieve these rates. Finally we place these results within the overall context of neural SDE training via efficient backpropagation.

\subsection{Pre-Processing Recombination Algorithm}
Here we detail the pre-processing recombination algorithm, from which cubature paths and weights are produced that allow for advantageous approximation complexity rates. First some definitions.

\begin{definition}[Localization \cite{litterer2012high}] Consider a discrete probability measure $\mu$ on $\reals^N$ and a collection $(U_j)_{j=1}^l$ of balls of radius $u$ on $\reals^N$ which covers the support of $\mu$. Then, construct a collection of positive measures $\mu_j, 1\leq j \leq l$ such that $\mu_i \perp \mu_j\,\, \forall i\neq j$ (disjoint supports), $\mu = \sum_{i=1}^l \mu_i$, and $\text{supp}(\mu_j) \subseteq U_j\cap\text{supp}(\mu)$. Such a collection $(U_j,\mu_j)_{j=1}^l$ is called a localization of $\mu$ to the cover $(U_j)_{j=1}^l$, with radius $u$. 
\end{definition}

\begin{definition}[Measure Reduction \citep{litterer2012high}]
\label{def:redmeas}
Let $P_r$ denote a set of $r$ test functions $\{p_1,\dots,p_r\}$ on a measure space $(\Omega, \mu)$ with $\mu$ a finite discrete measure:
\begin{equation}
\label{eq:mrd}
    \mu = \sum_{i=1}^{\hn}\lambda_i\delta_{z_i},\,\, \lambda_i > 0, z_i \in\Omega
\end{equation}
A probability measure $\tilde{\mu}$ is a reduced measure w.r.t. $\mu$ and $P_r$ if $\text{supp}(\tilde{\mu}) \subseteq \text{supp}(\mu)$, $\int p(x) \tilde{\mu}(dx) = \int p(x) \mu(dx)\,\, \forall p\in P_r$, $\text{card}(\text{supp}(\tilde{\mu})) \leq r+1$.

\end{definition}

\begin{definition}[Reduced Measure w.r.t. Localization \citep{litterer2012high}]
A measure $\tilde{\mu}$ is a reduced measure with respect to the localization $(U_j,\mu_j)_{j=1}^l$ and a finite set of integrable test functions $P$ if there exists a localization $(U_j,\tilde{\mu}_j)_{j=1}^l$ of $\tilde{\mu}$ such that for $1 \leq j \leq l$ the measures $\tilde{\mu}_j$ are reduced measures (Def.~\ref{def:redmeas}) w.r.t. $\mu_j$ and $P$. 
\end{definition}

Computing a reduced measure with respect to a localization is a key component of our pre-processing procedure (Algorithm~\ref{alg:wscr}). It allows one to transform a finite measure into another with \textit{drastically reduced support} which approximates the statistics of the original. \cite{litterer2012high} presents an algorithmic procedure for computing such reduced measures. We refer to this operation which takes a measure $\mu$ \eqref{eq:mrd} and outputs a reduced measure with respect to a localization $(U_j,\mu_j)_{j=1}^l$, with $\text{card}\{\text{supp}(\tilde{\mu})\} = l(r+1)$ as RM Procedure (RMP). For our purposes we may consider only this abstract procedure; for brevity we provide all details of RMP in Appendix~\ref{ap:mrd}. 

\begin{definition}[Reduced Measure Procedure (RMP)]
\label{def:rmp}
A Reduced Measure Procedure (RMP) takes as input a discrete measure $\mu$ \eqref{eq:mrd} and a localization $(U_j,\mu_j)_{j=1}^l$ of $\mu$ with radius $u$ and w.r.t. test functions $P_r$. It outputs a reduced measure $\tilde{\mu}$ w.r.t. these specifications. We write
$\tilde{\mu} = \RMP(\mu,(U_j,\mu_j)_{j=1}^l,P_r)$
\end{definition}

Now we introduce Algorithm~\ref{alg:wscr}, a recombination algorithm which exploits RMP to drastically reduce the number of ODEs which need to be solved within the cubature approximation.

\begin{algorithm}[h]
\begin{algorithmic}
\State Initialize time partition $\CD = 0 = t_0 < t_1 < \dots < t_k = T$ of $[0,T]$, and let $s_i = t_i - t_{i-1}$.
\State Take $r>0$; Initialize $P_r$ as a basis for the
\textit{space of polynomials} on $\reals^{\xdim}$ with degree at most $r$. 
\State Initialize paths $\om_1,\dots,\om_{\cn} \in C_{bv}^0([0,T],\reals^{\bdim})$ and weights $\lambda_1,\dots,\lambda_{\cn} > 0$ defining a cubature formula on Wiener space of degree $m > 0$, as in Section~\ref{sec:cpwc}. 
\State For discrete measure $\mu = \sum_{j=1}^l \mu_j \delta_{x_j}$ on $\reals^{\xdim}$, let 
\vspace{-0.2cm}
\[\KLV(\mu,s) := \sum_{j=1}^l \sum_{i=1}^
{\cn}\mu_j \lambda_i \delta_{x_j + \om_{i}(s)}\]
\vspace{-0.4cm}
\State Initialize point mass $x$ at the origin of $\reals^{\xdim}$, and set
 $\tQ_{\CD,u}^{(1)}(x) = \KLV(\delta_x, s_1)$ 
\For{$i=2:k-1$}
\State $\Q_{\CD,u}^{(i)}(x) = \KLV(\tQ_{\CD,u}^{(i-1)}(x), s_i)$; take any localization $(U_j,\mu_j)_{j=1}^{l_i}$ of $\Q_{\CD,u}^{(i)}(x)$ with radius $u_i$.
 \State  $\tQ_{\CD,u}^{(i)}(x) = \RMP(\Q_{\CD,u}^{(i)}(x), (U_j,\mu_j)_{j=1}^{l_i},P_r)$. (Recall Definition~\ref{def:rmp} for the reduced measure procedure ($\RMP$))\;
\EndFor
   \State $\tQ_{\CD,u}^{(k)}(x) = \KLV(\tQ_{\CD,u}^{(k-1)}(x), s_k)$\; \\
\For{$i=1:k$}
\State Express $\,\tQ_{\CD,u}^{(i)}(x) = \sum_{j=1}^{l_i(r+1)}\alpha_{\tilde{x}_j}\delta_{\tilde{x}_j}$
 \For{$j=1:l_i(r+1)$}
       \State Take $z\in[\cn^k]$ s.t. $\om_z(t_i) = \tilde{x}_j$ (exists by construction). Denote $\tilde{\lambda}_{I_z[i]} := \alpha_{\tilde{x}_j}$.
    \EndFor
    \State For the remaining $z \in[\cn^k]$ s.t. there is no $\tilde{x}_j \in \text{supp}(\tQ_{\CD,u}^{(i)}(x))$ with $\tilde{x}_j = \om_z(t_i)$, set $\tilde{\lambda}_{I_z[i]}=0$
\EndFor
\State Return weights $\{\tilde{\lambda}_{I_z[i]}, z\in[r^k],i\in[k]\}$
\caption{Pre-Processing for Wiener Space Cubature Recombination}\label{alg:wscr}
\end{algorithmic}
\end{algorithm}




Algorithm~\ref{alg:wscr} is a \textit{pre-processing} procedure which must only be completed once before training the neural SDE \eqref{eq:nsde}. Specifically, once the weights $\tilde{\lambda}_{I_z[i]}$ are recovered, we may construct for \textit{any $\theta \in \Theta$} the "recombined" cubature approximation
\begin{equation}
\label{eq:rcps}
    \tPhi_{\CLd}^{\theta} := \sum_{z\in [\cn^k]}\CLd(\phisoln)\tlam_{I_z[0]}\dots\tlam_{I_z[k]}
\end{equation} where, recall $\phisoln \in \CX$ is the solution of the ODE \eqref{eq:ode}. That is, we solve ODEs along weighted cubature paths \eqref{eq:pathcon} with weights constructed as above from the recombination procedure. \textit{By the recombination procedure, \textit{most} of the sequences of weights $\tlam_{I_z[0]}\dots\tlam_{I_z[k]}$ will equal zero, and thus ODEs will only need to be solved along a small fraction of the original paths $\om_z$.} Thus, \eqref{eq:rcps} forms a \textit{computationally efficient} approximation of $\CE[\CLd(\Xt)]$.

\subsection{Complexity Guarantees: Recombined Cubature}

The approximation error of the recombined cubature sum \eqref{eq:rcps} in terms of the number of paths is quantified in the following result, which is the second main result of this paper.

\begin{theorem}[Recombined Cubature Path Approximation Error]
\label{thm:bd2}
Let $t_i = T\left( 1-\left( 1-\frac{i}{k}\right)^{\gamma}\right)$ for $0\leq i\leq k$. Take $m,\gamma,r$ such that $r\ps \geq m \geq \gamma + 1$. Construct weights $\{\tilde{\lambda}_{I_z[i]}\}_{z\in[\cn^k]}$ from Algorithm~\ref{alg:wscr} using any localization with radius $u_i = s_i^{\ps/{2\gamma}}$. Then, under \ref{as:vecbd}, \ref{as:intfun}, \ref{as:uh}, and letting \eqref{eq:rcps} 
\[\tPhi_{\CLd}^{\theta} := \sum_{z\in [\cn^k]}\CLd(\phisoln)\tlam_{I_z[0]}\dots\tlam_{I_z[k]}\]
be constructed from the output of Algorithm~\ref{alg:wscr},
we have the following approximation guarantee:
\begin{align}
\begin{split}
\label{eq:cbound}
 \left|\CE[\CLd(\Xt)] - \tPhi_{\CLd}^{\theta}\right| = \mathcal{O}\left(n^{\frac{-\gamma\left(
    r\ps/\cfac - r\ps  + 1\right)}{\xdim(\ps-2)}} \right)
\end{split}
\end{align}
where $n$ is the number of ODE computations, and $\cfac$ is such that $\gamma \leq \frac{\ps(r+1)}{(\ps r)/\cfac + 1}$. The explicit constant factor in \eqref{eq:cbound} is given in \eqref{eq:interbd}.
\end{theorem}
\begin{proof}
See Appendix
\end{proof}



A result of Theorem~\ref{thm:bd2} is that we may achieve an effective approximation complexity improvement from Monte Carlo methods. The following Corollary illustrates a choice of reasonable (for which there exist explicit cubature path and weight constructions) parameters which produces a $\CO(n^{-1})$ rate.

\begin{corollary}[Improved Estimate Efficiency]
\label{cor:noneap}
For $\ps,\xdim \in \nat$, take $\gamma = 0.6$, $m = 5$, $c=0.6$, and $r = \frac{\xdim(\ps-2)-1}{0.4\,\ps}$. 
Then
\begin{equation}
\label{eq:noneap}
\left|\CE[\CLd(X)] - \tPhi_{\CLd}^{\theta} \right| = \CO\!\left(n^{-1}\right)
\end{equation}
where $n$ is the number of ODE solves as in Theorem~\ref{thm:bd2}.  
This improves on the Monte Carlo rate $\CO(n^{-1/2})$ and the quasi-Monte Carlo rate $\CO((\log n)/n)$.  
For $m=5$, explicit Wiener space cubature paths and weights are given in~\cite{lyons2004cubature}.
\end{corollary}

\subsubsection{Bound Discussion: Effective Achievable Rates}
Let us discuss these bounds. The precise bound with scaling factors \eqref{eq:boundprec} is given by
\begin{equation}\label{eq:bd2-kappa}
    \left|\CE[\CLd(\Xt)] - \tPhi_{\CLd}^{\theta}\right|
    \;\le\; C\,\Big(\frac{n}{r+1}\Big)^{-\kappa(r,\ps,\cfac)},
    \quad
    \kappa(r,\ps,\cfac) := \frac{\ps(r+1)}{\xdim(\ps-2)}\cdot
    \frac{\cfac - \ps r(\cfac-1)}{\ps r + \cfac},
\end{equation}
where the exponent $\kappa$ is obtained by taking $\gamma$ at its admissible maximum
$\gamma_{\max} = \frac{\ps(r+1)}{(\ps r)/\cfac + 1}$ in~\eqref{eq:cbound}.
In particular, $\kappa$ is an increasing function of $r$ whenever $\cfac<1$, 
and attains its maximum at $r=0$ whenever $\cfac\ge 1$. The bound in~\eqref{eq:bd2-kappa} involves both the exponent $\kappa(r,\ps,\cfac)$
and the \emph{effective sample size} (sample size after recombination) $n/(r+1)$.  
Formally, for $\cfac<1$ one has
\[
    \lim_{r\to\infty} \kappa(r,\ps,\cfac) = +\infty,
\]
so that the \emph{exponent} in~\eqref{eq:bd2-kappa} can be made arbitrarily large
by increasing the recombination precision $r$.  

However, one must also take into account the effect of the $r+1$ denominator in \eqref{eq:bd2-kappa}, as we discuss in the following remark.


\emph{Optimal rates under recombination precision $r$.}  
For fixed computational budget $n$, the base in \eqref{eq:bd2-kappa} penalizes large $r$. So, although Theorem~\ref{thm:bd2} suggests that arbitrarily fast decay
$\mathcal{O}(n^{-k})$ in the formal exponent is possible by taking $r\to\infty$
(when $\cfac<1$), the presence of $n/(r+1)$ in~\eqref{eq:bd2-kappa} means that,
for a fixed number $n$ of ODE computations,
the optimal choice of $r$ is finite. Thus, Corollary~\ref{cor:noneap} should be read as a formal $n^{-1}$ rate for fixed $r$, and for any fixed cubature path construction there will be a necessary limit to this precision $r$. We do not investigate this trade-off further here, but it will make an interesting point of future study to determine the optimal $r$-balanced rates that can be achieved under certain computational budgets and cubature constructions.


\textit{Dimension dependence.} Observe that the rate \eqref{eq:noneap} mirrors the dimension-independence of the Monte Carlo rate exponent:
by choosing $r$ large enough to suppress recombination bias, one recovers an
error decay in $n$ whose \emph{exponent} is independent of $\xdim$,
even though the constant prefactor may still grow with $\xdim$. 

\vspace{-0.2cm}

\subsubsection{Pre-Processing Complexity} Corollary~\ref{cor:noneap} demonstrates that under certain \\ parametrizations, we may achieve a convergence rate exceeding traditional Monte Carlo or quasi-Monte Carlo simulation. However, this is not free; there is a tradeoff between approximation efficiency \eqref{eq:noneap} and the complexity of the pre-processing operation (Algorithm~\ref{alg:wscr}) which constructs the cubature paths providing this rate. The following Corollary quantifies the complexity bound for this \textit{pre-processing} operation Algorithm~\ref{alg:wscr}\footnote{This is a consequence of Corollary~\ref{cor:ll12} and the bound \eqref{eq:nbound}}:

\begin{corollary}[Pre-Processing Complexity]
\label{cor:preproc}
    Algorithm~\ref{alg:wscr} operates in 
    \begin{equation}
    \label{eq:precomplex}
        \CO\left(r k^{\xdim(\ps - 2)} + rk^{\xdim(\ps/2 - 1)}\log(k^{\xdim(\ps/2 - 1)}/r)C(r+2,r+1)\right)
    \end{equation}
    time. Here $C(r+2,r+1)$ is the time to solve a system of linear equations in $r+2$ variables and $r+1$ constraints.
\end{corollary}

In the context of neural SDE training, we must incur the additional complexity \eqref{eq:precomplex} to construct recombined cubature weights, but then benefit from improved efficiency, e.g. \eqref{eq:noneap}, for every iterative gradient evaluation in the optimization scheme. 

The advantage of this approach will vary by application, but in many instances this pre-processing requirement should be negligible compared to the savings of the rate \eqref{eq:noneap}, especially in cases that require complex and expensive training procedures in moderate dimensions. We show in Section~\ref{sec:numerics} that this pre-processing step incurs negligible time complexity when compared to the efficiency gains even in high dimensions. 

Next we relate these results back to the overall framework of neural SDE training, by providing a numerical study which validates our theoretical efficiency improvements.

\section{Numerical Study}
\label{sec:numerics}

We now provide numerical experiments that validate the theoretical developments of 
Sections~\ref{sec:cubform}--\ref{sec:hor}. 
The focus is on two complementary aspects:
\begin{enumerate}[label=\roman*)]
\item Sample efficiency of cubature estimators relative to Monte--Carlo baselines. Recall Corollary~\ref{cor:noneap} predicts faster convergence rates for the cubature estimator, with respect to the number of paths $n$. Section~\ref{sec:pathconvg} compares this rate of convergence for a simple test functional. 
\item Practical wall--clock and memory benefits when cubature evaluation is used for full-pipeline 
neural SDE training. In Section~\ref{sec:traincomplex}, we fix $n$ and \textit{train} a neural SDE with respect to the standard variational loss functional \citep{li2020scalable}; we demonstrate the computational and memory benefits of cubature evaluations, by exploiting efficient deterministic ODE integration.
\end{enumerate}
We show that the cubature method can significantly outperform Monte Carlo stochastic simulation in both regimes, and that the pre-computation requirement of cubature path construction is negligible compared to these efficiency gains in training. 
All experiments are run on a Mac M3 CPU without GPU acceleration\footnote{Since both ODE and SDE solvers 
benefit comparably from GPU parallelization, the CPU-based comparison is without loss of generality.}, and all code can be found at \texttt{\small https://github.com/LukeSnow0/Neural-SDE}.

\subsection{Convergence of Loss Functional Estimates}
\label{sec:pathconvg}

Theorem~\ref{thm:bd2} quantifies the approximation rates which are attainable by cubature methods; here we numerically investigate these rates.
We compare cubature and Monte--Carlo estimates of the expected path functional value $\CE_{\theta}[\CLd(X^{\theta})]$, with the following simple test functional satisfying \ref{as:intfun}:
\vspace{-0.2cm}
\begin{equation}
\label{eq:testfun}
\CLd(X^{\theta}) \;=\; \int_0^1 \bigl( X_t^{\theta} - \sin(2\pi t)\bigr)^2 \,dt,
\end{equation}
$\Xt$ is the solution of the Stratonovich neural SDE of Section~\ref{sec:cubform}, with randomly initialized and \textit{fixed} drift and diffusion neural network parametrizations $\theta$. \eqref{eq:testfun} is an arbitrary Lipschitz-smooth test functional that we use for ease of demonstration.\footnote{In Section~\ref{sec:traincomplex} we implement a full-pipeline neural SDE training procedure using the realistic variational loss functional \citep{li2020scalable}, to demonstrate the performance within a state-of-the-art training framework.}

\begin{figure}[H]

    \centering
    \includegraphics[width=0.6\linewidth]{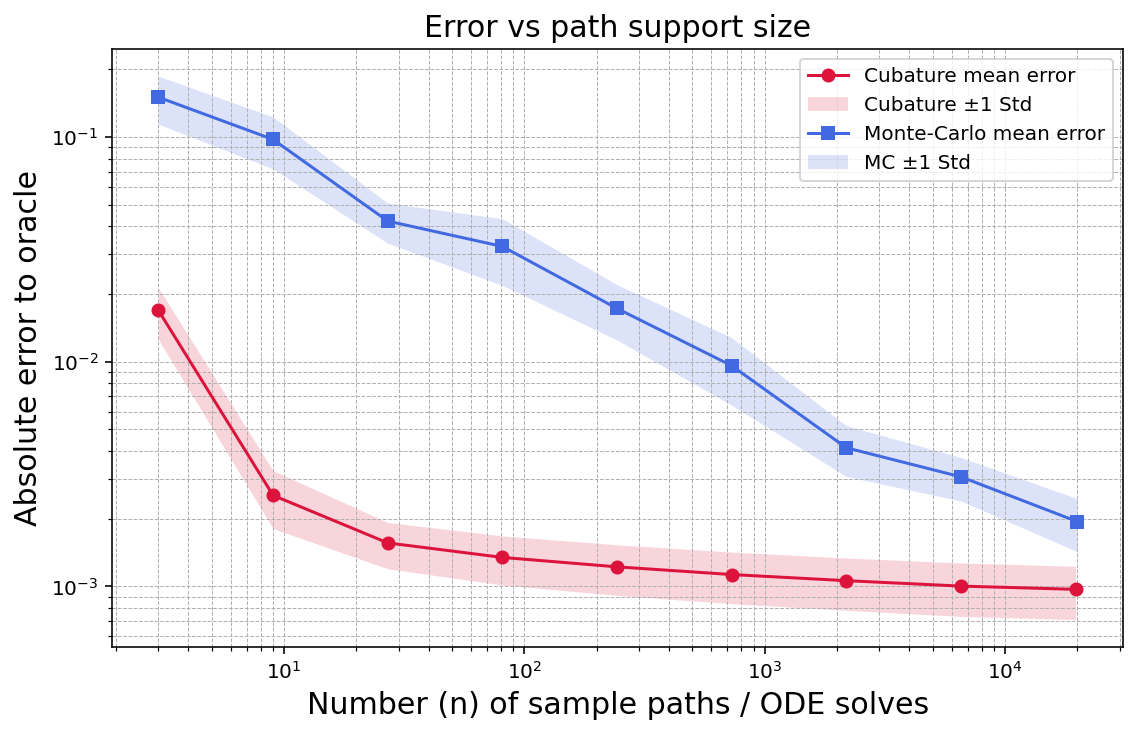}
    \caption{\small 
    Convergence of cubature (degree--$5$) vs.\ Monte--Carlo estimates of $\CLd(X)$.  
    Monte--Carlo exhibits the standard $O(n^{-1/2})$ decay in path count $n$, 
    while cubature achieves faster convergence consistent with the $O(n^{-1})$ rate of Corollary~\ref{cor:noneap}.  
    For equal $n$, cubature estimates are uniformly more accurate.
    }
    \label{fig:loss_convergence}
\end{figure}
We compare against a baseline high resolution oracle $\CE[\CLd(X^{\theta})] = \CE\left[ \int_0^1 \bigl( X_t^{\theta} - \sin(2\pi t)\bigr)^2 \,dt\right]$ computed using Monte Carlo simulation with Euler-grid resolution of 4000 points, and a batch of 10e5 sample paths.  
Cubature estimates \eqref{eq:rcps} of this oracle use degree--$5$ paths\footnote{See Appendix~\ref{sec:o5paths} for these path constructions.} constructed via Algorithm~\ref{alg:wscr} with $\gamma=0.6$, $r=4$, $m=5$. Let $\tPhi_{\CLd,n}^{\theta}$ denote this recombined cubature approximation \eqref{eq:rcps} with path support size $n$.  
Let $\hat{\CE}_{n}[\CLd(X^{\theta})]$ denote the Monte Carlo estimate of this oracle using the same grid resolution, under $n$ sample paths. Figure~\ref{fig:loss_convergence} displays the absolute errors $|\hat{\CE}_{n}[\CLd(X^{\theta})] - \CE[\CLd(X^{\theta})|$ (red) and $|\Phi_{\CLd,n}^{\theta} - \CE[\CLd(X^{\theta})|$ (blue) vs the number of sample paths $n$, on log-log scale. Observe that Monte--Carlo achieves the expected $O(n^{-1/2})$ rate, 
while cubature achieves a faster decay consistent with $O(n^{-1})$ to begin with, but then plateaus. \footnote{This plateau can be explained since for large $n$, residual error is dominated by ODE time discretization, not cubature approximation, 
so further optimization of the numerical solver should bring the observed rate closer to the theoretical guarantee.  There are some interesting numerical analyses which would be helpful for improving our understanding of the implementation-level complexity guarantees and trade-offs, quantifying such effects as this.} Moreover, observe that at equal path counts the cubature estimate yields strictly smaller constants than Monte--Carlo, so that accuracy is improved in the finite-sample setting before considering the asymptotic complexity rates \footnote{It should be of significant theoretical interest to predict when such finite-sample approximation advantage can be observed.}.

\subsection{Training Complexity: ODE Cubature vs.\ SDE Monte--Carlo}
\label{sec:traincomplex}
Here we compare the cubature and Monte Carlo estimators when used in the full-pipeline procedure of training a neural SDE, using the 
variational loss functional of \citep{li2020scalable}. \footnote{This variational loss functional is perhaps the most well-established and standard objective for neural SDE training. Details on this variational loss objective can be found in Appendix~\ref{ap:lossfun}.} The data paths $\Xd$ are generated from a multi-variate Stratonovich SDE \eqref{eq:strat} with fixed drift and diffusion. The first marginal dimension of these data paths is visualized in blue in Figure~\ref{fig:train_vis}. The model 
is trained with identical neural network architectures under the two gradient estimation techniques:

\begin{itemize}
\item[-] \textit{NSDE--MC:} Loss functional evaluation of SDE solutions with \texttt{torchsde.sdeint} from package \texttt{torchdiffeq}. \footnote{Each step thus requires Gaussian increments and diffusion algebra in addition to vector field evaluations.}. Gradients are computed using PyTorch automatic differentiation through these SDE solution paths. 
\item[-] \textit{Cubature--ODE:} Loss functional evaluation, produced by Algorithm~\ref{alg:wscr} with $m=3,\,\, k=5,\,\, r=4,\,\,\gamma = 0.6$, using \texttt{odeint} from package \texttt{torchdiffeq}. \footnote{Here per-step cost is limited to vector field evaluations since we eliminate the Brownian motion.}. Gradients are computed using PyTorch automatic differentiation through these ODE solution paths. 
\end{itemize}

\sisetup{
  table-number-alignment = center,
  round-mode             = places,
  round-precision        = 3,
  detect-weight          = true,
  detect-family          = true
}

\begin{figure}

    \centering
    \includegraphics[width=0.55\linewidth]{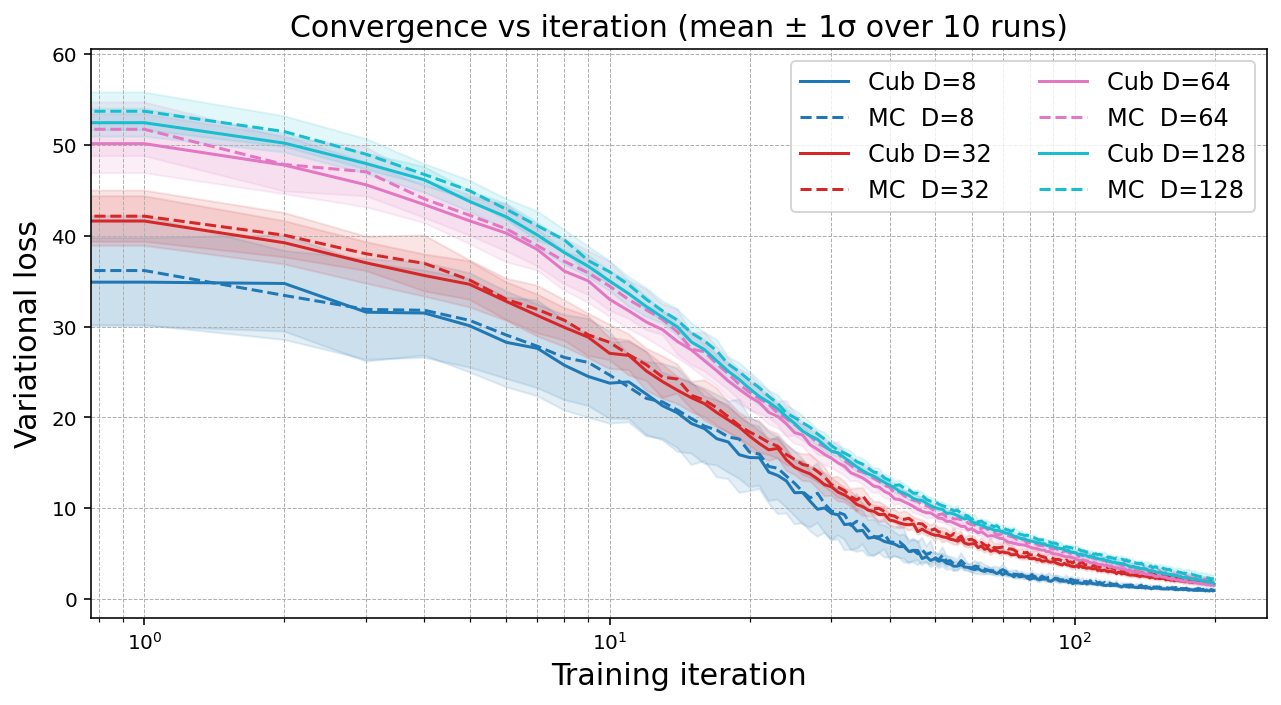}
    \caption{\small 
     Convergence of the variational loss of a neural SDE vs. training iteration, in dimensions $8,32,64,128$, using both cubature and Monte Carlo (MC) estimators. We observe that the optimization over the standard loss functional \eqref{eq:objective} using traditional Monte Carlo, and the optimization over the cubature loss functional \eqref{eq:sum}, behave equivalently in each dimension when measured with respect to training iterations.
    }
    \label{fig:trainitr}
\end{figure}

\begin{figure}

    \centering
    \includegraphics[width=0.55\linewidth]{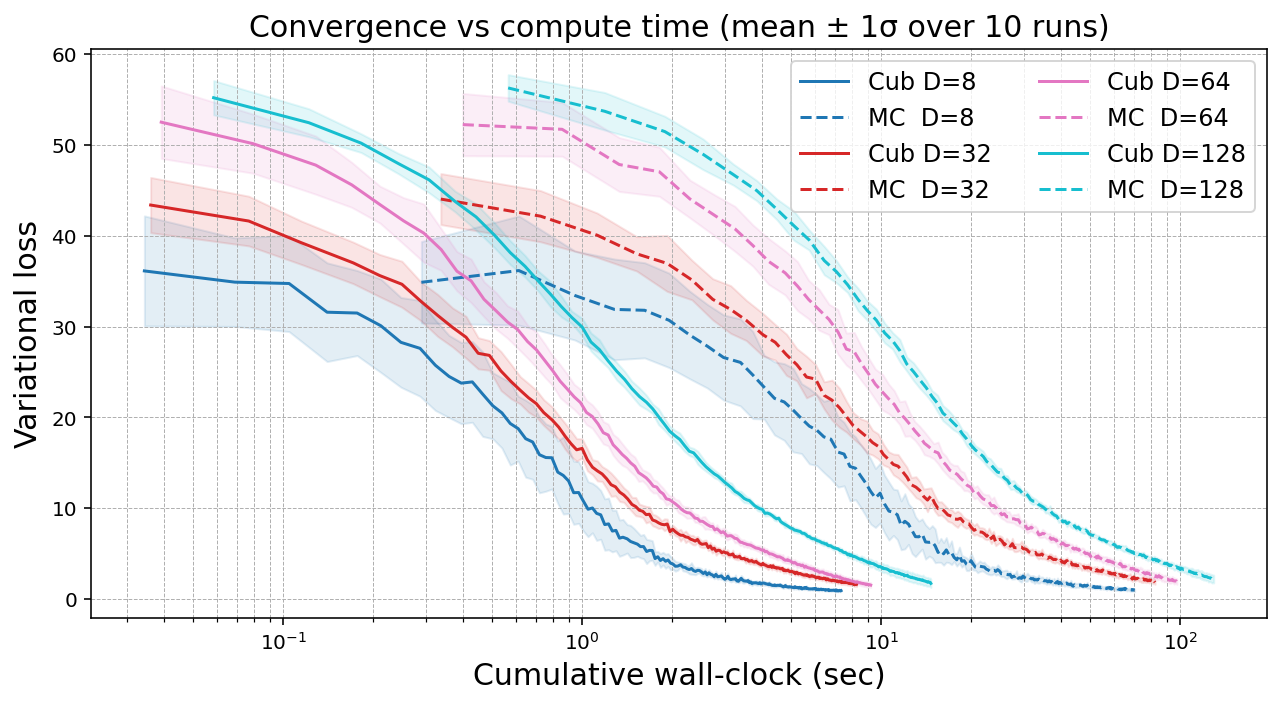}
    \caption{\small 
     Convergence of the variational loss of a neural SDE during training, in dimensions $8,32,64,128$, using both cubature and Monte Carlo (MC) estimators. The cubature method evaluates an estimate of the expected loss functional, at each training epoch, using ODE solutions w.r.t. deterministic cubature paths. The MC method evaluates this estimate by standard stochastic approximation. It is observed that while both methods converge at similar rates w.r.t. the number of training epochs, the cubature estimator offers significant speedups (Table~\ref{tab:timings}), and thus the convergence vs. wall-clock execution time is much faster in the cubature method. 
    }
    \label{fig:traintime}
\end{figure}

Table~\ref{tab:timings} reports mean per-epoch runtime and peak memory usage over $n=50$ path evaluations and sweeps in the dimension $D$, with discretization resolution $T=1000$. Observe that the cubature method dominates in low-to-moderate dimensions, in both wall-clock time and memory usage. The memory advantage remains stable in very high dimensions while the wall-clock time advantage begins to deteriorate past dimension $250$. Furthermore, the "overhead" of cubature path construction remains negligible compared to these efficiency gains.  

\begin{table}
\centering
\small
\caption{\small Per-epoch wall--clock time and memory usage vs dimension $D$}
\label{tab:timings}
\begin{tabular}{rrrrrrrr}
\toprule
\small  D &  MC time &  Cub time &  Speedup &  MC mem &  Cub mem &  Mem adv &  Overhead \\
\midrule
  4 &     2.689 s &    0.388 s &                       6.93x &                 1.87 MB &                0.65 MB &                             2.88x &                    0.0132 s \\
 8 &     2.647 s &    0.443 s&                       5.97x &                 1.87 MB &                0.65 MB &                             2.88x &                    0.0118 s\\
16 &     2.709 s &    0.488 s &                        5.55x &                 1.87 MB  &                0.65 MB &                             2.88x &                    0.0132 s \\
32 &     2.831 s &    0.603 s&                        4.70x &                 1.87 MB  &                0.65 MB &                             2.88x &                    0.0134 s\\
64 &    3.079 s &    1.142 s &                        2.70x &                 1.87 MB  &                0.65 MB &                             2.88x &                    0.0189 s \\
128 &    3.621 s &    2.398 s &                       1.51x &                 1.87 MB  &                0.65 MB &                             2.88x &                    0.0174 s \\
256 &    5.140 s &    5.248 s &                       0.98x &                 1.87 MB  &                0.65 MB &                             2.88x &                    0.0233 s \\
\bottomrule
\end{tabular}
\end{table}

Figure~\ref{fig:trainitr} displays the variational error convergence over training iteration, of both the cubature method and the standard Monte Carlo method. It can be observed that the convergence behavior is identical, validating the well-posedness of the cubature loss functional \eqref{eq:sum}. In particular, one can safely deploy cubature evaluations through loss functional \eqref{eq:sum} without concern that the loss landscape will change sufficiently to cause computational bottlenecks during training, even in high dimensions. 

Figure~\ref{fig:traintime} displays the variational error incurred by Monte Carlo evaluations and the cubature evaluations, with respect to \textit{wall-clock computational time}, over varying dimensions with fixed discretization grid $T=200$. We observe stable speedup with respect to dimension, in contrast to the report of Table~\ref{tab:timings}, reflecting the extra computational overhead of \textit{backpropagation} through stochastic Monte Carlo simulation paths\footnote{A more comprehensive quantification of this full-pipeline cubature advantage would be a worthwhile endeavor for future research.}. 
Figure~\ref{fig:train_vis} visualizes the resulting trajectories and training curves, over the training backpropagation iterations. We observe that both methods achieve comparable quality of data approximation, but cubature training is significantly faster, especially in low-to-moderate dimensions, yielding speedups of an order of magnitude. This matches the complexity analysis: cubature eliminates stochastic overhead (RNG and diffusion algebra), resulting in constant-factor benefits while also inheriting the $O(n^{-1})$ approximation convergence rate.\footnote{The experimental results in this section suggest that cubature methods can drastically improve computational efficiency, especially in low-to-moderate dimensions and with a moderate number of sample paths. However, the efficiency gains begin to deteriorate in high dimensions and path samples; further insight into this observation will require analysis of numerical implementation aspects of both procedures. However, the regimes in which we observe advantageous complexity are consistent with our theoretical predictions, in particular the formal approximation complexity of Corollary~\ref{cor:noneap} and the gradient backpropagation advantage incurred by non-stochastic ODE solution paths.}

\begin{figure}[H]
    \centering
    \includegraphics[width=\linewidth]{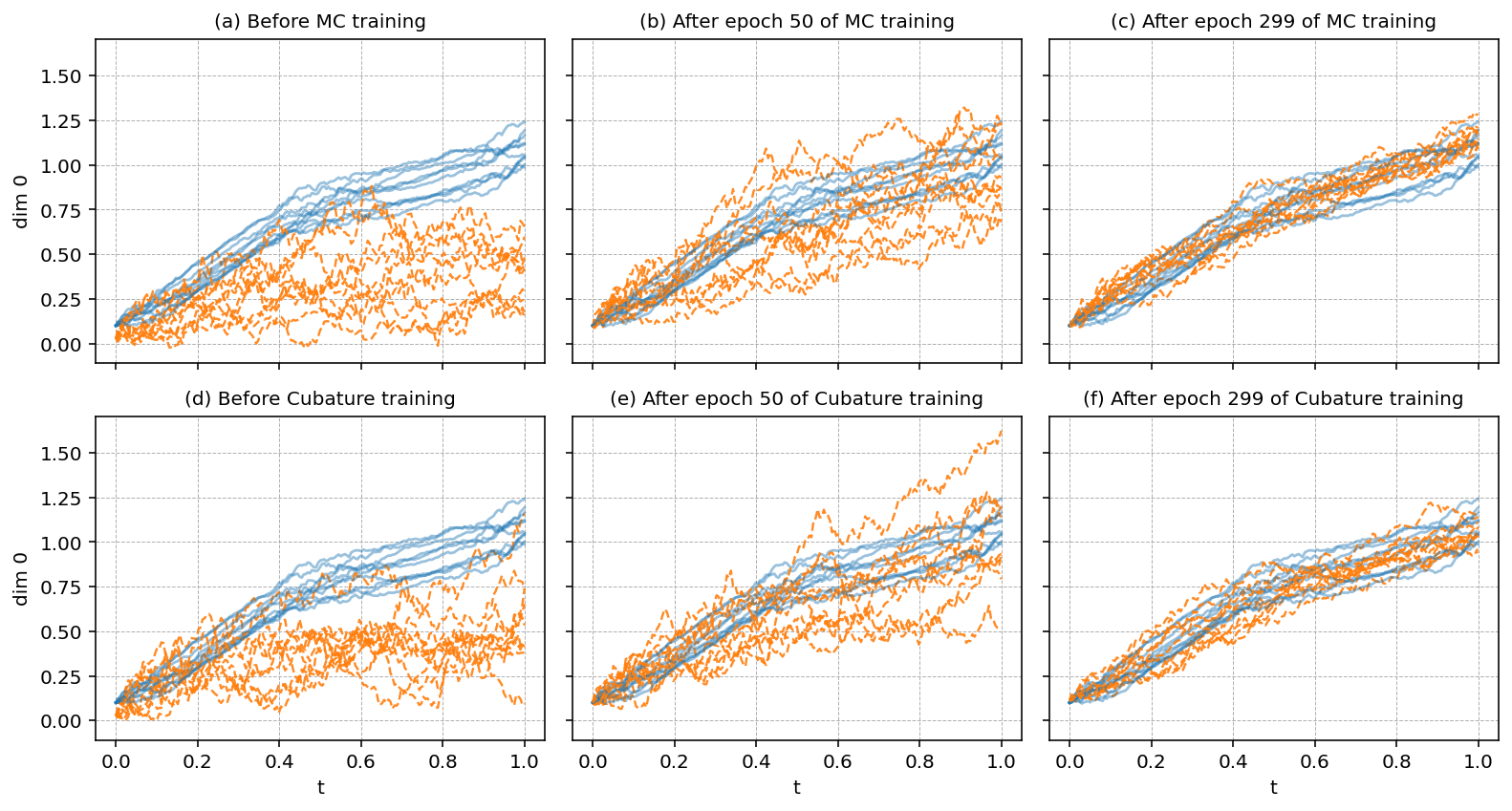}
    \caption{\small We train an 8-dimensional neural SDE to match the path dataset $\Xd$ displayed (via first marginal dimension) in blue. The orange paths represent (first marginal dimension) samples from the neural SDE distribution before, during, and after training via both cubature and standard MC methods. Observe that the data approximation is qualitatively comparable at each stage of the training process for each technique. The advantage of the cubature method is that the per-epoch compute time and memory requirements are made more efficient, as displayed in Figure~\ref{fig:traintime}.}
    \label{fig:train_vis}
\end{figure}

\vspace{-0.2cm}
\section{Conclusions}
We have provided a new framework for training neural SDEs by exploiting and extending the theory of Wiener space cubature. In particular, we have provided a constructive approach to approximating an expected loss functional of a neural SDE by a weighted combination of deterministic ODE solutions. Our main results rely on an extension of guarantees in Wiener-space cubature to the nonlinear functional domain, and our main mathematical contribution is precisely this extension. Our constructed algorithmic approach exploiting this advance proves advantageous for two reasons. First, we may compute gradients more efficiently by using deterministic ODE solvers. Second, we demonstrate that under reasonable assumptions we improve upon the Monte Carlo complexity of $\CO(n^{-1/2})$ to achieve a $\CO(n^{-1})$ approximation guarantee. This rate is obtained by carefully choosing algorithmic parameters; in future work we propose to explore the dependence of this rate on the underlying parameters. This will lead to an adaptable framework where the computational load can be optimized for specific generative modeling settings. Furthermore, we demonstrate the computational advantages of this approach in full-pipeline numerical implementations of neural SDE optimization, revealing stable computational speedups and memory efficiency gains even in high dimensions. However, our work is primarily theoretical and there are many interesting facets of numerical implementation which warrant further investigation.

\section{Appendix}

\subsection{Sufficient Conditions for the Uniform H\"ormander Property in Neural SDEs}
\label{sec:UHcond}

Consider the Stratonovich SDE
\[
\mathrm{d}X_t \;=\; V_0^\theta(t,X_t)\,\mathrm{d}t \;+\; \sum_{j=1}^m V_j^\theta(t,X_t)\circ\mathrm{d}W_t^j,
\]
with drift $V_0^\theta$ and diffusion vector fields $V_j^\theta$, where $\theta\in\Theta$ lies in a compact parameter set.
The uniform H\"ormander condition of order $s$ requires the existence of $s\in\mathbb{N}$ and $\mu>0$ such that
\[
\sum_{|\alpha|\le s}\big\langle V^\theta_{[\alpha]}(t,x),\,\xi\big\rangle^2 \;\ge\; \mu
\quad\text{for all } (t,x,\theta)\in[0,T]\times\mathbb{R}^d\times\Theta,\ \|\xi\|=1,
\]
where $V^\theta_{[\alpha]}$ are iterated Lie brackets of the fields $\{V_i^\theta\}_{i=0}^m$. 
Below we list practical sufficient conditions tailored to neural SDE parameterizations.

\begin{theorem}[Sufficient Conditions for Assumption~\ref{as:uh}]
\label{thm:suff-UH}
Suppose each $V_i^\theta$ is $C^{s+1}$ in $x$, with derivatives up to order $s+1$ uniformly bounded in $(t,\theta)$. 
Then the uniform H\"ormander condition holds under either of the following structural assumptions:

1. \textbf{Uniform Ellipticity.} 
The diffusion matrix $g_\theta(t,x)=[V_1^\theta\;\cdots\;V_m^\theta]\in\mathbb{R}^{d\times m}$ has full rank $d$ and is uniformly nondegenerate:
\[
g_\theta(t,x)\,g_\theta(t,x)^\top \;\succeq\; \lambda I_d,
\qquad \forall (t,x,\theta),
\]
for some $\lambda>0$. In this case $s=1$ suffices.

2. \textbf{Bracket-Generating (Hypoelliptic) Case.} 
If $m<d$ and the diffusion matrix is degenerate, assume there exists $s\in\mathbb{N}$ and $\mu>0$ such that
\[
\sum_{j=1}^m\sum_{k=0}^{s-1}\big\langle \mathrm{ad}_{V_0^\theta}^k V_j^\theta(t,x),\,\xi\big\rangle^2 \;\ge\;\mu,
\quad \forall (t,x,\theta),\ \|\xi\|=1,
\]
where $\mathrm{ad}_{V_0^\theta}^k V := [V_0^\theta,[V_0^\theta,\dots,[V_0^\theta,V]\dots]]$ denotes $k$-fold iterated commutators.
\end{theorem}

The above conditions can be satisfied by common neural SDE parameterizations:
\begin{enumerate}
\item \emph{Elliptic diffusion networks.} 
Take $m\ge d$ and parameterize 
\[
g_\theta(t,x) \;=\; L_\theta \,\tilde g_\theta(t,x),
\]
with $L_\theta\in\mathrm{GL}(d)$ constrained so that $\sigma_{\min}(L_\theta)\ge \lambda_0>0$, and $\tilde g_\theta$ bounded above and below (e.g.\ via positive activations such as Softplus).

\item \emph{Kolmogorov-type chains.} 
Partition $x=(u,v)$, place diffusion only on $u$ (so $V_j^\theta=\partial_{u_j}$), and design $V_0^\theta$ so that $\nabla_u V_0^\theta$ couples $u$ into $v$ with uniformly bounded coefficients. Then $[V_0^\theta,V_j^\theta]$ spans the missing $v$-directions.

\item \emph{Control-affine structures.} 
If locally $f_\theta(t,x)\approx A_\theta(t,x)x+b_\theta(t,x)$ and $g_\theta(t,x)\approx B_\theta(t,x)$, a uniform Kalman-type rank condition on $(A_\theta,B_\theta)$---namely that $\{B_\theta,A_\theta B_\theta,\dots,A_\theta^{d-1}B_\theta\}$ spans $\mathbb{R}^d$ uniformly---implies UH.
\end{enumerate}

Smooth activations (e.g.\ $\tanh$, Softplus, SiLU) ensure the required $C^{s+1}$ regularity. 
Parameter compactness (via weight decay, spectral normalization, or explicit box constraints) ensures uniformity in $\theta$. 
ReLU activations yield only piecewise $C^1$ vector fields, which complicates the theory, and are best avoided if strict H\"ormander conditions are required.

\subsection{Variational Loss Objective}
\label{ap:lossfun}

We provide here a complete derivation of the variational loss functional used in our simulations in Section~\ref{sec:numerics}, following \cite{li2020scalable}. The objective arises from introducing a variational posterior process to approximate the latent path distribution of the neural SDE.

\textit{Latent Neural SDE Model}:
We assume the latent state $z_t \in \mathbb{R}^d$ evolves according to the Itô SDE
\begin{equation}
dz_t = f_\theta(t, z_t)\,dt + g_\theta(t, z_t)\,dW_t, 
\qquad z_0 \sim p_\theta(z_0),
\label{eq:latent_SDE}
\end{equation}
where $f_\theta : [0,T]\times\mathbb{R}^d \to \mathbb{R}^d$ and $g_\theta : [0,T]\times\mathbb{R}^d \to \mathbb{R}^{d\times d_b}$ are neural networks, and $W_t$ is a $d_b$-dimensional standard Brownian motion.  
Observations are generated via a conditional decoder:
\begin{equation}
y_t \mid z_t \sim p_\theta(y_t \mid z_t), \qquad t=0,\dots,T.
\end{equation}
This defines the generative distribution $p_\theta(y_{0:T}, z_{0:T})$ over observed and latent paths.

\textit{Variational Posterior Process}:
To approximate the true posterior distribution of $z_{0:T}$ given $y_{0:T}$, we introduce a variational diffusion process of the form
\begin{equation}
dz_t = \tilde f_\phi(t, z_t)\,dt + g_\theta(t, z_t)\,dW_t,
\qquad z_0 \sim q_\phi(z_0),
\label{eq:var_post}
\end{equation}
where $\tilde f_\phi$ is a neural network (the variational drift).  
Note that the diffusion $g_\theta$ is \emph{shared} between the prior \eqref{eq:latent_SDE} and the variational posterior \eqref{eq:var_post}. This choice ensures that the Radon–Nikodym derivative between the two path measures is tractable.

\textit{Evidence Lower Bound (ELBO)}:
Let $q_\phi(z_{0:T})$ denote the path measure induced by \eqref{eq:var_post}. The variational training objective is the evidence lower bound (ELBO):
\begin{equation}
\mathcal{L}(\theta,\phi) 
= \mathbb{E}_{q_\phi(z_{0:T})}
\Big[ \log p_\theta(y_{0:T} \mid z_{0:T}) \Big]
- \mathrm{KL}\!\left(q_\phi(z_{0:T}) \,\big\|\, p_\theta(z_{0:T})\right).
\label{eq:elbo}
\end{equation}
The first term encourages accurate reconstruction of the data, while the second term penalizes deviation from the prior dynamics \eqref{eq:latent_SDE}.

\textit{KL Divergence via Girsanov}:
Since both \eqref{eq:latent_SDE} and \eqref{eq:var_post} share the same diffusion $g_\theta$, Girsanov’s theorem gives the Radon–Nikodym derivative between the two measures:
\begin{align*}
&\frac{d q_\phi}{d p_\theta}(z_{0:T})
\\&= \exp\!\Bigg(
\int_0^T \!\!
\big\langle g_\theta^{-1}(t,z_t)\big(f_\theta(t,z_t)-\tilde f_\phi(t,z_t)\big),\, dW_t \big\rangle
- \frac{1}{2}\int_0^T \!\!
\big\| g_\theta^{-1}(t,z_t)(f_\theta(t,z_t)-\tilde f_\phi(t,z_t)) \big\|^2 dt
\Bigg).
\end{align*}
Taking the logarithm and expectation under $q_\phi$, the KL divergence simplifies to
\begin{equation}
\mathrm{KL}(q_\phi \,\|\, p_\theta)
= \frac{1}{2}\, \mathbb{E}_{q_\phi}\!\left[
\int_0^T \big\| g_\theta^{-1}(t,z_t)\,
\big(f_\theta(t,z_t)-\tilde f_\phi(t,z_t)\big) \big\|^2 \,dt \right].
\label{eq:KL}
\end{equation}

\textit{Final Objective}:
Combining \eqref{eq:elbo} and \eqref{eq:KL}, the explicit variational loss functional is
\begin{equation}
\mathcal{L}(\theta,\phi)
= \mathbb{E}_{q_\phi(z_{0:T})}\!\left[
\log p_\theta(y_{0:T} \mid z_{0:T})
- \tfrac{1}{2}\int_0^T \big\| g_\theta^{-1}(t,z_t)\,
\big(f_\theta(t,z_t)-\tilde f_\phi(t,z_t)\big) \big\|^2 dt
\right].
\label{eq:final_loss}
\end{equation}
Thus, the training objective decomposes into: a data reconstruction term $\mathbb{E}_{q_\phi}[\log p_\theta(y_{0:T}\mid z_{0:T})]$, and a drift-regularization term given by the quadratic penalty \eqref{eq:KL}.

In the context of Section~\ref{sec:numerics}, evaluating the expectation in \eqref{eq:final_loss} can be done either by Monte--Carlo sampling of paths from \eqref{eq:var_post}, or by cubature approximation (Theorem~\ref{thm:bd2}), yielding the improved $O(n^{-1})$ rate of Corollary~\ref{cor:noneap}. It is straightforward to evaluate the objective \eqref{eq:final_loss} by Monte Carlo simulation using sample paths of the variational SDE. Next we show precisely how to evaluate this using the cubature paths.

\subsection*{Evaluating the Variational Loss with Cubature Paths} 

Let $\{(\omega^{(i)},\lambda_i)\}_{i=1}^n$ be a degree-$m$ cubature path set with weights $\lambda_i$.  
On each cubature path, we solve the controlled ODE (Stratonovich form):
\[
\frac{d}{dt} z_t^{(i)} \;=\; \tilde f_\phi(t,z_t^{(i)}) \;+\; g_\theta(t,z_t^{(i)})\,\dot\omega^{(i)}(t),
\qquad z_0^{(i)} \sim q_\phi(z_0),
\]
where $\dot\omega^{(i)}(t)$ denotes the (piecewise constant or linear) derivative of the deterministic cubature path $\omega^{(i)}$.  For each trajectory $z^{(i)}$ we evaluate:
\begin{align*}
\text{Reconstruction:}\quad
R^{(i)} &= \int_0^T \log p_\theta(y_t \mid z_t^{(i)})\,dt
\quad \text{(or discrete sum over observation times)},\\
\text{Regularization:}\quad
K^{(i)} &= \frac{1}{2}\int_0^T 
\big\|g_\theta^{-1}(t,z_t^{(i)})\big(f_\theta(t,z_t^{(i)})-\tilde f_\phi(t,z_t^{(i)})\big)\big\|^2\,dt.
\end{align*}

\textit{Cubature estimator}:  
Using the notation \eqref{eq:rcps}, the cubature approximation to the loss functional is then the weighted sum:$\tPhi_{\CLd}^{\phi,\theta}
= \sum_{i=1}^n \lambda_i \,\big( R^{(i)} - K^{(i)} \big).
$
and the data-samples are approximated minimizing $\tPhi_{\CLd}^{\phi,\theta}$ over $\phi,\theta$.

\subsection{Order-5 Cubature Path Examples}
\label{sec:o5paths}
Here we illustrate one-dimensional degree--$5$ KLV cubature paths satisfying \eqref{def:cpw}. These paths are explicitly constructed in \cite{lyons2004cubature} Section 5; for brevity we refer to this source for their derivation and mathematical formulation. Figure~\ref{fig:CubOrd} shows cubature paths for discretization orders $k=1,\dots,5$, constructed via Algorithm~\ref{alg:wscr}, with $\gamma = 2$, $r = 4$, $m=5$.

The support size (number of paths) $n$ grows exponentially with $k$ if paths are concatenated naively, but as explained 
in Section~\ref{sec:hor}, the recombination procedure ensures that $n$ grows only polynomially, 
enabling the improved $O(n^{-1})$ approximation rate of Corollary~\ref{cor:noneap}. By Figure~\ref{fig:ctime}, we see that even for $k=10$ ($n=4612$ paths), the full construction takes only $0.8$ seconds. Thus, this overhead of cubature path generation should be negligible relative to the cost and complexity of training in the vast majority of cases. However, this trade-off may need to be analyzed further for specific use-cases with varying dimensions and computational budgets.

\begin{figure}[H]
    \centering
    \begin{subfigure}[b]{0.3\textwidth}
        \includegraphics[width=\linewidth]{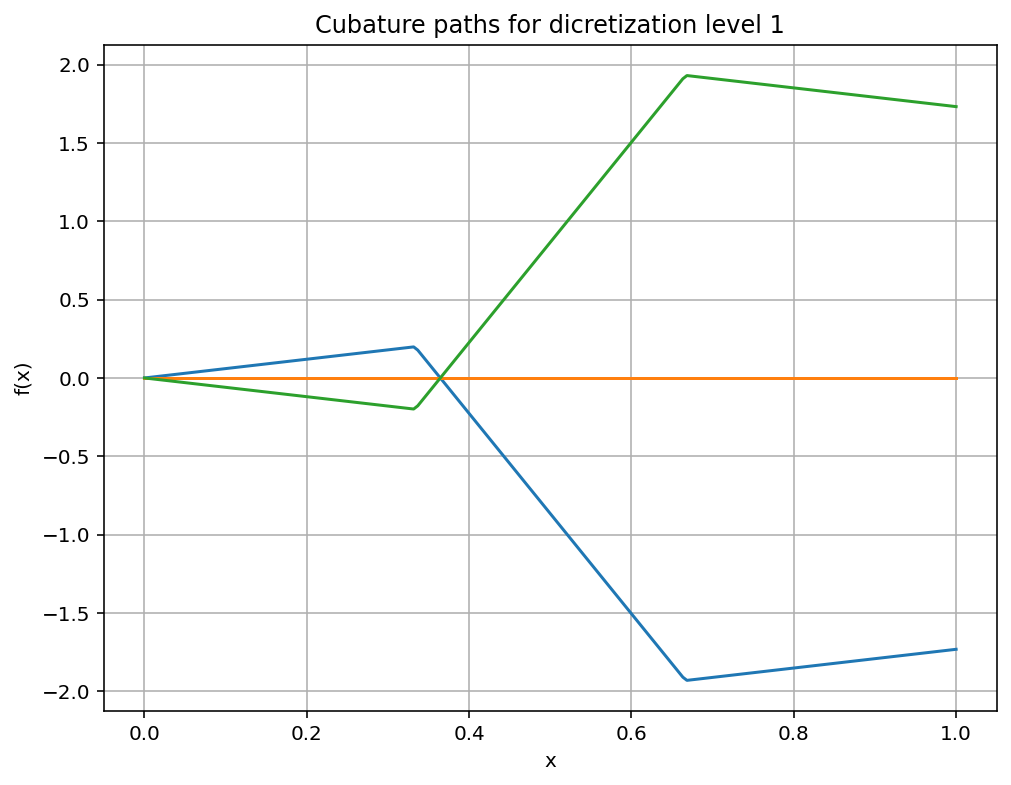}
        \caption{$k=1, n=3$}
    \end{subfigure}
    \hfill
    \begin{subfigure}[b]{0.3\textwidth}
        \includegraphics[width=\linewidth]{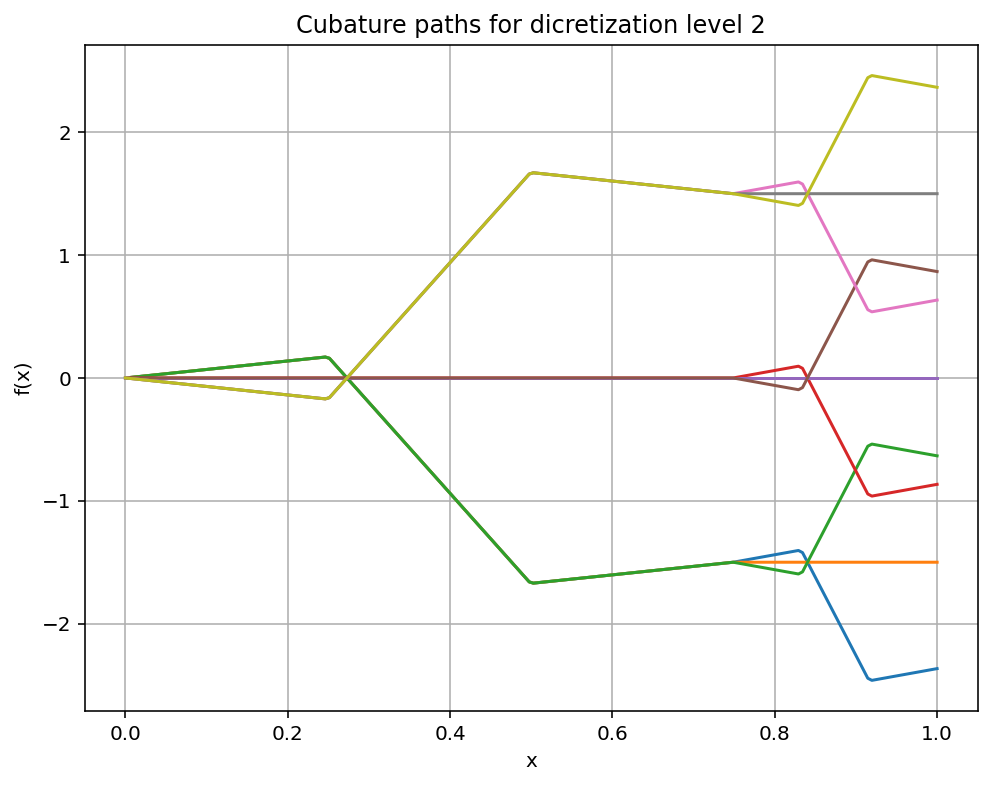}
        \caption{$k=2, n=9$}
    \end{subfigure}
    \hfill
    \begin{subfigure}[b]{0.3\textwidth}
        \includegraphics[width=\linewidth]{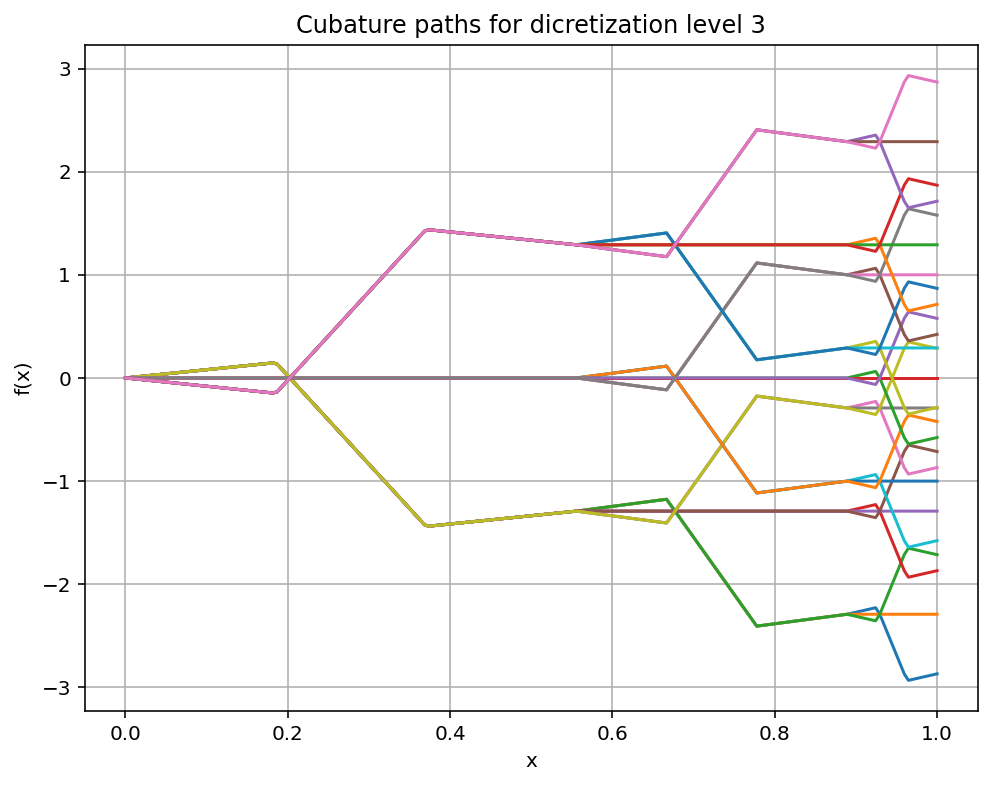}
        \caption{$k=3, n=27$}
    \end{subfigure}
    
    \vspace{1em}
    
    \begin{subfigure}[b]{0.3\textwidth}
        \includegraphics[width=\linewidth]{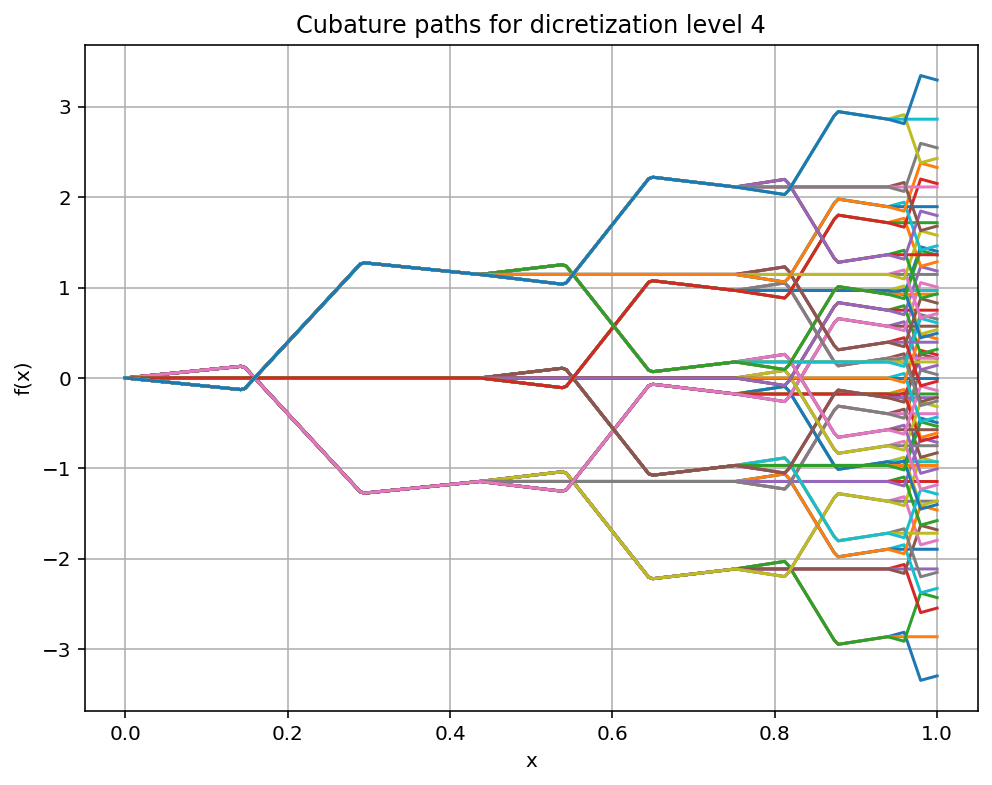}
        \caption{$k=4, n=81$}
    \end{subfigure}
    \hfill
    \begin{subfigure}[b]{0.3\textwidth}
        \includegraphics[width=\linewidth]{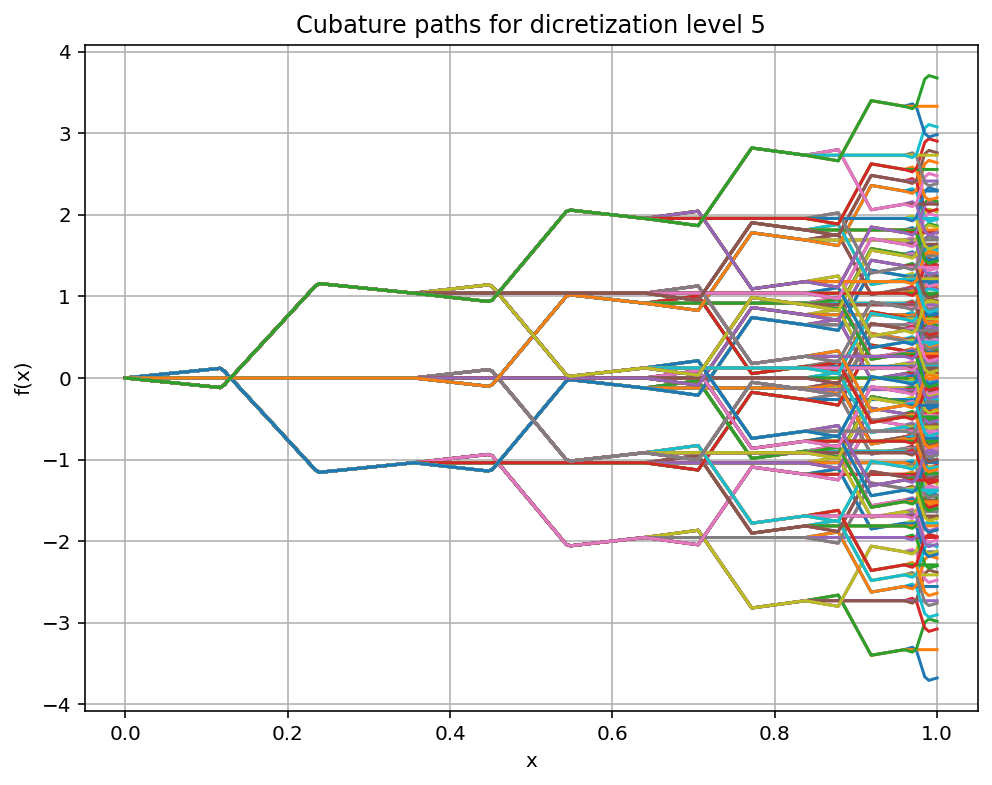}
        \caption{$k=5, n=243$}
    \end{subfigure}
    \hfill
    \begin{subfigure}[b]{0.3\textwidth}
        \includegraphics[width=\linewidth]{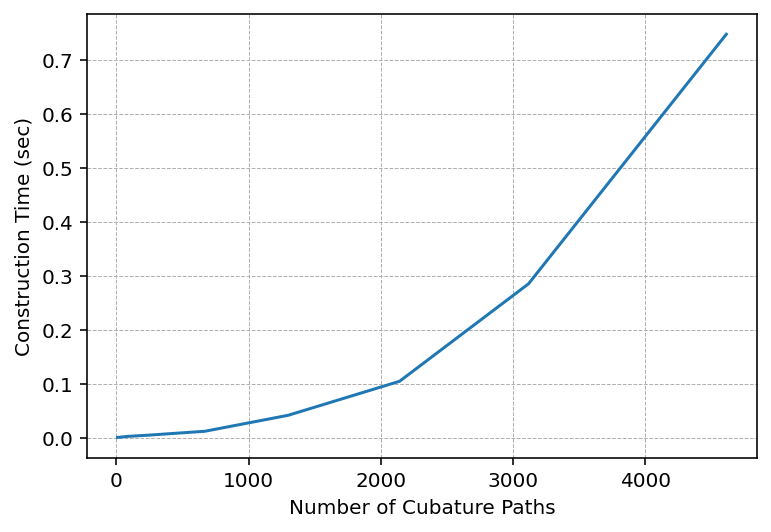}
        \caption{Construction time vs.\ $n$}
        \label{fig:ctime}
    \end{subfigure}    
    
    \caption{\small 
    Degree--$5$ cubature paths for discretization levels $k=1,\dots,5$.  
    The number of paths $n$ grows exponentially in $k$ without recombination, 
    but Section~\ref{sec:hor} shows how recombination yields polynomial growth.  
    Panel (f) shows that the one-time preprocessing cost remains negligible ($0.8$s for $k=10, n = 4612$).  
    This validates the feasibility of using cubature paths in training while retaining the improved $O(n^{-1})$ rate 
    of Corollary~\ref{cor:noneap}, which improves training efficiency.
    }
    \label{fig:CubOrd}
\end{figure}

\subsection{Measure Reduction Details}
\label{ap:mrd}
This appendix provides a detailed explanation of the recombination framework of Section~\ref{sec:hor}.

Consider the finite set of test functions $P_r = \{p_1,\dots,p_r\}$ on $(\Omega,\mu)$, a measure space with $\mu$ a finite discrete measure 
\[\mu = \sum_{i=1}^{\hn}\lambda_i \delta_{z_i}, \,\, \lambda_i > 0, z_i \in \Omega\]

Let $P$ be a $\reals^r$-valued random variable $P := (p_r,\dots,p_r)$ defined on $(\Omega,\mu)$. Then the law $\mu_P$ of $P$ is the discrete measure on $\reals^r$:
\[\mu_P =\sum_{i=1}^{\hn} \lambda_i\delta_{x_i},\,\,\, x_i = (p_i(z_i),\dots,p_r(z_i))^T \in \reals^r\]
and the center of mass (CoM) for measure $\mu_P$ is given as 
$CoM(\mu_P) = \sum_{i=1}^{\hn}\lambda_i x_i$.

The key insight is that to construct a \textit{reduced measure} $\tilde{\mu}_P$
w.r.t. $\mu_P$ and $P_r$ it is sufficient to find a subset $x_{i_k}$ of the points $x_i$ and positive weights $\tilde{\lambda}_{i_k}$ to produce a new probability measure $\tilde{\mu}_P = \sum \tilde{\lambda}_{i_k} \delta_{x_{i_k}}$ s.t. $CoM(\tilde{\mu}_P) = CoM(\mu_P)$. Then a reduced measure can be constructed as 
\[\tilde{\mu} = \sum \tilde{\lambda}_{i_k}\delta_{z_{i_k}}\]
with $z_{i_k} \in \text{supp}(\mu)$ satisfying $P(z_{i_k}) = x_{i_k}$. 

We now provide a concrete algorithmic procedure for computing a reduced measure with respect to a localization. We first introduce a single-step reduction iteration, as follows. 

\begin{algorithm}[H]
\begin{algorithmic}
   \State Input: $n$ points $\{x_{i}\}_{i=1}^n$ and weights $\{\lambda_i\}_{i=1}^n$
   \State Solve the system of equations: $ \sum_{i=1}^{n} u_i x_i = 0,\quad \sum_{i=1}^n u_i = 0$
\State Compute $\alpha = \min_{ u_i>0} \frac{\lambda_i}{u_i},\, \lambda_i' = \lambda_i - \alpha u_i$ \algorithmiccomment{$\lambda_i' = 0$ for some $i$}
   \State Set $i' = \{i : \lambda_i' = 0\}$ 
\While {$i \leq n-1$}
\If {$i < i'$}
 \State $\tilde{\lambda}_i = \lambda_i'$
 \Else
 \State $\tilde{\lambda}_i = \lambda_{i+1}'$, \,
 $\tilde{x}_i = x_{i+1}$\;
\EndIf
\EndWhile
\State Output weights $(\tilde{\lambda}_i)_{i=1}^{n-1}$, points $(\tilde{x}_i)_{i=1}^{n-1}$ 
\caption{Reduction Iteration \citep{litterer2012high}}\label{alg:redone}
\end{algorithmic}
\end{algorithm}

Then, applying Algorithm~\ref{alg:redone} iteratively in the following Algorithm~\ref{alg:recom} allows us to obtain a reduced measure 
\begin{algorithm}[H]
\begin{algorithmic}
\caption{Recombination \citep{litterer2012high}}\label{alg:recom}
\State Input: $\mu_P = \sum_{i=1}^{\hn} \lambda_i \delta_{x_i}$
\State Partition the support of $\mu_P$ into $2(r+1)$ sets $I_j, 1\leq j \leq 2(n+1)$ of equal size.
\State Compute $\nu = \sum_{i=1}^{2(r+1)}\nu_i \delta_{\hx_i}$, where 
$\hx_j = \CE_{\mu_P}(x | x\in I_j) = \sum_{x_i\in I_j} \frac{\lambda_i x_i}{\nu_j}$
and $\nu_j = \mu_P(I_j) = \sum_{i: x_i \in I_j} \lambda_i$
\State Apply Algorithm~\ref{alg:redone} $r+1$ times beginning with weights $\{\nu_i\}_{i=1}^{2(r+1)}$ and $\{\hat{x}_i\}_{i=1}^{2(r+1)}$, to produce a measure $\tilde{\nu} = \sum_{j=1}^{r+1}\tilde{\nu}_{i_j}\delta_{\tilde{x}_{i_j}}$ with CoM($\tilde{\nu}$) = CoM($\nu$). 
\State Repeat 1 - 4 with $\mu_P = \sum_{j=1}^{r+1}\sum_{x_k \in I_{i_j}} \tilde{\nu}_{i_j} \frac{\lambda_k}{\nu_{i_j}}\delta_{x_k}$ until $r+1$ particles remain.\;
\end{algorithmic}
\end{algorithm}

\begin{lemma}[\cite{litterer2012high}]
Algorithm~\ref{alg:recom} requires $\lceil \log (\hat{n}/r)\rceil$ iterations of Algorithm~\ref{alg:redone}.
\end{lemma}

\begin{corollary}[\cite{litterer2012high}]
\label{cor:ll12}
 Algorithm~\ref{alg:recom} operates in 
 \[O(r\hn + r\log(\hn/r)C(r+2,r+1))\]
 time, where $C(r+2,r+1)$ is the time to compute the solution of a system of linear equations in $r+2$ variables and $r+1$ constraints. 
\end{corollary}

\subsection{Proofs}

Let us first provide a technical definition that will be used in the proofs. Associated to the structural dynamics of the neural SDE \eqref{eq:nsde} is a Markov semigroup. 

\begin{definition}[Markov Semigroup]
    The Markov semigroup $P_t^{\theta}$ associated to the dynamics of the SDE \eqref{eq:nsde} is given by 
    \[P_t^{\theta}f(x) := \CE(f(X_t^{\theta}) | X_0^{\theta} =x)\]
\end{definition}

We also provide the following technical result which will be used in our main proof. 
\subsubsection*{Stone--Weierstrass Density of Regularized Cylinder Functionals}

Let \\$\Omega := C([0,T]; \mathbb{R}^d)$ denote the path space equipped with the uniform norm $\|\omega\|_\infty := \sup_{t \in [0,T]} \|\omega(t)\|_1$. We consider a class of \emph{cylinder functionals} $\mathcal{A}$ of the form:
\[
F_n(\omega) := f_n(\omega(t_1), \dots, \omega(t_n)), \quad 0 \leq t_1 < \dots < t_n \leq T,
\]
where $f_n : (\mathbb{R}^d)^n \to \mathbb{R}$ is measurable and satisfies the regularized Lipschitz condition:
\begin{equation}
\label{eq:avg-lip}
|f_n(x_1,\dots,x_n) - f_n(y_1,\dots,y_n)| \leq \frac{\lip}{n} \sum_{i=1}^n \|x_i - y_i\|_1, \quad \forall x_i, y_i \in \mathbb{R}^d.
\end{equation}

We define the algebra of such regularized cylinder functionals as:
\[
\mathcal{A} := \bigcup_{n \in \mathbb{N}} \mathcal{F}_n, \quad \text{where } \mathcal{F}_n := \left\{ F_n(\omega) = f_n(\omega(t_1),\dots,\omega(t_n)) \ \middle|\ f_n \text{ satisfies } \eqref{eq:avg-lip} \right\}.
\]

We aim to show that $\mathcal{A}$ is dense in $\mathcal{C}(\Omega)$, the space of real-valued continuous functionals on $\Omega$ endowed with the supremum norm topology.

\begin{theorem}[Density of Regularized Cylinder Functionals]
\label{thm:stone-weierstrass}
The algebra $\mathcal{A}$ of regularized cylinder functionals is uniformly dense in $\mathcal{C}(\Omega)$ restricted to compact subsets $K \subset \Omega$. In particular, for any $F \in \mathcal{C}(K)$ and $\varepsilon > 0$, there exists $F_n \in \mathcal{A}$ such that:
\[
\sup_{\omega \in K} |F(\omega) - F_n(\omega)| < \varepsilon.
\]
\end{theorem}

\begin{proof}
We verify the assumptions of the classical Stone--Weierstrass theorem \citep{de1959stone}.

\textit{(1) Algebra:} Let $F_n, G_n \in \mathcal{F}_n$ with corresponding functions $f_n, g_n$. Then $F_n + G_n$ and $F_n \cdot G_n$ also belong to $\mathcal{F}_n$. The sum is clearly Lipschitz with the same bound, and for the product:
\[
|f_n(x)g_n(x) - f_n(y)g_n(y)| \leq |f_n(x)||g_n(x) - g_n(y)| + |g_n(y)||f_n(x) - f_n(y)|,
\]
which is controlled by the Lipschitz condition \eqref{eq:avg-lip} and boundedness on compact domains.

\textit{(2) Separates points:} For any $\omega \neq \omega' \in \Omega$, there exists $t_0 \in [0,T]$ and a coordinate $j \in \{1,\dots,d\}$ such that $\omega(t_0)^j \neq \omega'(t_0)^j$. Define:
\[
F(\omega) := \omega(t_0)^j,
\]
which is a cylinder function in $\mathcal{F}_1$ and satisfies \eqref{eq:avg-lip}, hence $F \in \mathcal{A}$ distinguishes $\omega$ from $\omega'$.

\textit{(3) Contains constants:} The constant function $F(\omega) := c$ belongs to $\mathcal{F}_1$ and satisfies \eqref{eq:avg-lip} trivially.

Therefore, $\mathcal{A}$ is a subalgebra of $\mathcal{C}(\Omega)$ that contains constants and separates points. By the Stone--Weierstrass theorem, $\mathcal{A}$ is uniformly dense in $\mathcal{C}(K)$ for every compact $K \subset \Omega$.

\end{proof}

\subsubsection{Theorem~\ref{thm:cubap}}
\label{sec:pf3}
\begin{proof}

We proceed in three parts.
\subsubsection*{Part 1}
Let $F: [0,T] \times \reals^{\xdim}$ be the class of functions satisfying: $\sup_{t\in[0,T]}\|\nabla \lf(t,\cdot)\|_{\infty} := \lfgsup < \infty$; i.e., $F = \text{Lip}_{\lfgsup}$. We start by proving the following bound, which will serve as a key step in obtaining the main bound \eqref{eq:cpae}. 

\begin{equation}
\label{eq:ftbd}
    \sup_{t\in[0,T], f\in F}\left|\sum_{z\in[q^k]}\lambda_z'f(t,\phi_z^{\theta}(t)) - \CE[f(t,X_t^{\theta})]\right| \leq 3\,K \sup_{t\in[0,T]}\|\nabla f(t,\cdot)\|_{\infty}\left(s_k^{1/2} + \sum_{i=1}^{k-1}\frac{s_i^{(m+1)/2}}{(T-t_i)^{m/2}}\right)
\end{equation}

where $K > 0$ is a constant.

 We begin by expressing the term $f(t,X_t^{\theta})$  as the following stochastic Taylor expansion (see \cite{lyons2004cubature} Proposition 2.1). Denoting $X_{t,x}^{\theta}$ the stochastic process $X_t^{\theta}$ initialized at point $x$, we have
\begin{align*}
   & f(t,X_{t,x}^{\theta}) = \sum_{(i_1,\dots,i_k)\in\CA_m} V_{i_1}^{\theta}\dots V_{i_k}^{\theta} f(t,x)\int_{0 < t_1 < \dots < t_k < t} \circ dB_{t_1}^{i_1}\dots \circ dB_{t_k}^{i_k} + R_m(t,x,f) \\
    & \text{where }
    \sup_{x\in\reals^{\xdim}}\sqrt{\CE(R_m(t,x,f))^2)} \leq C t^{(m+1)/2}\sup_{(i_1,\dots,i_k)\in\CA_{m+2}\backslash \CA_m}\|V_{i_1}^{\theta}\dots V_{i_k}^{\theta}f(t,\cdot)\|_{\infty}
\end{align*}
and $C$ is a constant depending only on $\bdim$ and $m$. Take cubature weights and paths as defined in Section~\ref{sec:cpc}, and define at time $T>0$ the Wiener space measure 
\[Q_T = \sum_{j=1}^{\cn} \lambda_j \delta_{\om_T,j}\] 
such that 
\[\CE_{Q_T}(f(T,X_{T,x}^{\theta})) = \sum_{j=1}^{\cn}\lambda_j f(T,\phi_{T,x}^{\theta}(\om_{T,j}))\]
where $\phi_{T,x}^{\theta}(\om_{T,j})$ denotes the time $T$ solution of ODE \eqref{eq:ode} with initial condition $x$ and path $\om_{T,j}$
Then, for $t \leq T$ we have 
\begin{align}
\begin{split}
\label{eq:tTbound}
    &|(\CE - \CE_{Q_T})(f(t,X_{t,x}^{\theta}))|\\
    &\quad \leq \CE(|R_m(t,x,f)|) + \CE_{Q_T}(|R_m(t,x,f)|) \\
    &\quad \quad + \left|(\CE - \CE_{Q_T})\left(\sum_{(i_1,\dots,i_k) \in\CA_m}V_{i_1}^{\theta}\dots V_{i_k}^{\theta} f(t,x)\int_{0<t_1 <\dots < t_k < t} \circ dB_{t_1}^{i_1} \dots \circ dB_{t_k}^{i_k}\right) \right|
\end{split}
\end{align}
By Lemma 3.1 of \cite{lyons2004cubature} we see that
\[\sup_{x\in\reals^{\xdim}}\CE_{Q_T}[|R_m(t,x,f)|] \leq C_{\bdim,m} t^{(m+1)/2}\sup_{(i_1,\dots,i_k)\in\CA_{m+2}\backslash\CA_m}\|V_{i_1}^{\theta}\dots V_{i_k}^{\theta}f(t,\cdot)\|_{\infty}\]
and the third term in \eqref{eq:tTbound} can be upper bounded by the expansion:
\begin{align}
\begin{split}
\label{eq:expan}
    &\biggl|(\CE - \CE_{Q_t})\left(\sum_{(i_1,\dots,i_k) \in\CA_m}V_{i_1}^{\theta}\dots V_{i_k}^{\theta} f(t,x)\int_{0<t_1 <\dots < t_k < t} \circ dB_{t_1}^{i_1} \dots \circ dB_{t_k}^{i_k}\right)  \\
    &\quad + (\CE_{Q_t} - \CE_{Q_T})\left(\sum_{(i_1,\dots,i_k) \in\CA_m}V_{i_1}^{\theta}\dots V_{i_k}^{\theta}f(t,x)\int_{0<t_1 <\dots < t_k < t} \circ dB_{t_1}^{i_1} \dots \circ dB_{t_k}^{i_k}\right) \biggr| \\
    & = \sum_{j=1}^{q}\lambda_j \sum_{(i_1,\dots,i_k) \in\CA_m}V_{i_1}^{\theta}\dots V_{i_k}^{\theta} f(t,x) \int_{0<t_1 <\dots < t_k < t} (d\om_{t,j}^{i_1}(t_1)\dots d\om_{t,j}^{i_k}(t_k) - d\om_{T,j}^{i_1}(t_1)\dots d\om_{T,j}^{i_k}(t_k)) \\
    &\leq 2\tilde{C}_{\bdim,m} t^{(m+1)/2}\sup_{(i_1,\dots,i_k)\in\CA_m}\|V_{i_1}^{\theta}\dots V_{i_k}^{\theta} f(t,\cdot)\|_{\infty}
\end{split}
\end{align}
where $\tilde{C}_{\bdim,m}$ is a constant depending only on $\bdim$ and $m$.
In \eqref{eq:expan} the first equality follows since
the first term vanishes by definition of the Wiener space cubature paths, and the bound follows as in the proof of Lemma 3.1 in \cite{lyons2004cubature}. 

Thus, putting these together into \eqref{eq:tTbound}, we obtain the bound
\begin{align}
\begin{split}
\label{eq:supbd}
\
&\sup_{x\in\reals^{\xdim}}\biggl| \CE\left(f(t,X_{t,x}^{\theta}) - \sum_{j=1}^{\cn}\lambda_j f(t,\phi_{t,x}^{\theta}(\om_{T,j})) \right) \biggr| \\
& \quad \leq C t^{(m+1)/2}\left( \sup_{(i_1,\dots,i_k)\in\CA_{m+2}\backslash \CA_m}\|V_{i_1}^{\theta}\dots V_{i_k}^{\theta} f(t,\cdot)\|_{\infty} + 2\sup_{(i_1,\dots,i_k)\in\CA_m}\|V_{i_1}^{\theta},\dots,V_{i_k}^{\theta} f(t,\cdot)\|_{\infty}\right) 
\end{split}
\end{align}
for some $C$ dependent only on $\bdim$ and $m$.

Consider the construction $ \Phi_{\lf}(t) = \sum_{z\in [{\cn}^k]}\lf(t,\phi_{t,x}^{\theta}(\om_z))\lambda_{I_z[0]}\dots\lambda_{I_z[k]}$. Let \\$\ut = \arg\min_{t_i < t, i\in[k]}|t-t_i|$ and $\ui$ be the index of $\ut$. Identify $t$ with $t_{\ui+1}$. Let $\us_i = t_i - t_{i-1}$ for all $i \leq \ui$, $
\us_{\ui + 1} = (t-\ut)$. Then, defining the Markov random variable $(\Psi_{i}^{\theta})_{1\leq i\leq \ui+1}$ by $\Psi_0^{\theta} = y$,
\[\PR(\Psi_{i}^{\theta} = \phi_{\us_i,x}^{\theta}(\om_{s_i,j})|\Psi_{i-1}  =x) = \lambda_j,\quad i\in 1,\dots,\ui+1\]
observe that $\CE_{\Psi^{\theta}}\lf(t,\Psi_{\ui+1}^{\theta}) = \Phi_f^{\theta}(t)$.
So, by Theorem 3.3 of \cite{lyons2004cubature} we may extend \eqref{eq:supbd} to $\Phi_f(t)$ to conclude that
\begin{align*}
\begin{split}
    &|\Phi_f^{\theta}(t) - \CE[f(t,X_t^{\theta})]|\\
    &\leq C\sum_{j=1}^{\ui + 1} \us_{j}^{(m+1)/2}\biggl\{ \sup_{(i_1,\dots,i_k)\in\CA_{m+2}\backslash \CA_m}\|V_{i_1}^{\theta}\dots V_{i_k}^{\theta} P_{t - t_j}^{\theta}f(t_j,\cdot)\|_{\infty} \\
    &\qquad \qquad \qquad \qquad \qquad + 2\sup_{(i_1,\dots,i_k)\in\CA_m}\|V_{i_1}^{\theta},\dots,V_{i_k}^{\theta} P_{t - t_j}^{\theta}f(t_j,\cdot)\|_{\infty}\biggr\} \\
    & \leq C\sum_{j=1}^{k} s_j^{(m+1)/2}\biggl\{ \sup_{\substack{(i_1,\dots,i_k)\in\CA_{m+2}\backslash \CA_m \\ t\in[0,T]}}\|V_{i_1}^{\theta}\dots V_{i_k}^{\theta} P_{T - t_j}^{\theta}f(t,\cdot)\|_{\infty} \\
    &\qquad \qquad \qquad \qquad \qquad + 2\sup_{\substack{(i_1,\dots,i_k)\in\CA_m \\ t\in[0,T]}}\|V_{i_1}^{\theta},\dots,V_{i_k}^{\theta} P_{T - t_j}^{\theta}f(t,\cdot)\|_{\infty}\biggr\}
\end{split}
\end{align*}
where the second inequality follows simply by monotonicity of the sum in $j$.
Now, under the \\H\"ormander condition \ref{as:uh} \citep{kusuoka1985applications}, 
\[\|V_{i_1}^{\theta}\dots V_{i_k}^{\theta} P_s^{\theta} f\|_{\infty} \leq \frac{Ks^{1/2}}{s^{(k+\text{card}\{j,i_j=0\})/2}}\|\nabla f\|_{\infty}\]
for some constant $K$; this can be applied to both $\infty$-norm terms above. 
By the same derivation as in the proof of Proposition 3.6 of \cite{lyons2004cubature} we obtain for every $t\in[0,T]$  
\begin{align}
\begin{split}
\label{eq:tbd}
    |\Phi_f^{\theta}(t) - \CE[f(t,X_t^{\theta})]| &\leq 3\,K \sup_{t\in[0,T]}\|\nabla f(t,\cdot)\|_{\infty}\left(s_k^{1/2} + \sum_{i=1}^{k-1}\frac{s_i^{(m+1)/2}}{(T-t_i)^{m/2}}\right) \\
    &\leq 3\,K A\left(s_k^{1/2} + \sum_{i=1}^{k-1}\frac{s_i^{(m+1)/2}}{(T-t_i)^{m/2}}\right)\\
    &\leq 3\,K A \,C(m,\gamma)T^{1/2} k^{-\gamma/2}
\end{split}
\end{align}
where in the last inequality we use that by \cite{lyons2004cubature} we have, for $t_j = T\left(1 - \left(1-\frac{j}{k}\right)^{\gamma}\right)$ and $0 < \gamma < m-1$,
\begin{equation}
\label{eq:sumbd}
 \left(s_k^{1/2} + \sum_{i=1}^{k-1}\frac{s_i^{(m+1)/2}}{(T-t_i)^{m/2}}\right) \leq C(m,\gamma)T^{1/2} k^{-\gamma/2}
\end{equation}

\subsubsection*{Part 2}

 Let, for all $n\in\nat$, $\Fn$ be the class of multivariate functions $f_n: [\reals^{\xdim}, \dots,\reals^{\xdim}] \to \reals$, where the length of the input vector to $f_n$ is $n$. We suppose $f_n$ has the following $n$-regularized Lipschitz regularity:
\begin{equation}
|f_n(x_1,\dots,x_n) - f_n(y_1,\dots,y_n)| \leq \frac{\lip}{n}\sum_{i=1}^n \|x_i - y_i\|_1
\end{equation}
First some notation. Let $Z$ denote the joint random vector $(X_{t_1,x}^{\theta}, \dots, X_{t_n,x}^{\theta})$, for $t_{i} \in [0,T]$, and $\hat{Z}^j = (\phi_{t_1,x}^{\theta}(\om_{T,j}),\dots,\phi_{t_n,x}^{\theta}(\om_{T,j}))$. That is, $Z$ collects points at $t_1,\dots,t_n$ with respect to the neural SDE law. Denote $\mathbb{P}$ the law of $Z$. $\hat{Z}^j$ collects points w.r.t. the ODE solution \eqref{eq:ode} along the $j$'th cubature path $\om_{T,j}$. Then, we denote $\hat{Z}$ the random vector defined by the discrete measure over $\hat{Z}_j$ with weights $\lambda_j$, and $\hat{\mathbb{P}}$ the law of $\hat{Z}$.

Now, for the marginals. We denote by $\pi_i$ the law of $X_{t_i}$, and $\hat{\pi}_{t_i}$ the discrete measure with support on cubature points $\phi_{t_i,x}^{\theta}(\om_{T,j})$, with respective weights $\lambda_j$. Then, we may induce a coupling $\gamma_i \in \Gamma(\pi_i,\hat{\pi}_i)$ ($\Gamma(\pi_i,\hat{\pi_i})$ is the set of joint measures with marginals $\pi_i$ and $\hat{\pi_i}$), such that $\int_{\reals^{2\xdim}}\|x-y\|d\gamma_i(x,y) = W_1(\pi_i,\hat{\pi}_i)$ and $W_1$ denotes the 1-Wasserstein distance. By definition of the 1-Wasserstein distance, such a coupling exists and attains the infimum \citep{villani2008optimal} \[\gamma_i = \arg\inf_{\gamma \in \Gamma(\pi,\hat{\pi})}\int_{\reals^{2\xdim}}\|x-y\|d\gamma(x,y)\]
Now, we may induce the joint coupling as a product measure: $\gamma := \gamma_1 \otimes\dots\otimes\gamma_n$, and proceed as:
\begin{align}
\begin{split}
\label{eq:wbd}
    &|\CE(f(X_{t_1},\dots,X_{t_n}) - \sum_{j}\lambda_jf(\phi_{t_1,x}^{\theta}(\om_{T,j}),\dots,\phi_{t_n,x}^{\theta}(\om_{T,j}))| = |\CE_{\mathbb{P}}[f(z)] - \CE_{\tilde{P}}[f(z)]| \\& \quad = \left|\int (f(z) - f(\hat{z}))d\gamma(z,\hat{z})\right| \leq \int |f(z) - f(\hat{z})|d\gamma(z,\hat{z}) \leq \frac{\lip}{n}\int\sum_{i=1}^n\|x_{t_i} -y_{t_i} \|d(\gamma_1\otimes\dots\otimes\gamma_n) \\
    \\& \quad = \frac{\lip}{n}\sum_{i=1}^n \int\|x_i - y_i\|d\gamma_i(x_i,y_i) = \frac{\lip}{n}\sum_{i=1}^n W_1(\pi_i, \hat{\pi}_i)
\end{split}
\end{align}
Now we set this up in this form because we can immediately obtain a bound on $W_1(\pi_i,\hat{\pi}_i)$ by the 1-Wasserstein duality. Observe, by \eqref{eq:tbd} we have: 
\begin{align*}
     &\sup_{f \in \text{Lip}_A, \,t\in[0,T]}|\sum_{j=1}^k \lambda_j f(\phi_{t,x}^{\theta}(\om_{T,j}) - \CE[f(X_t^{\theta})]| \leq 3\,K A \,C(m,\gamma)T^{1/2} k^{-\gamma/2} \\
     &\Rightarrow \sup_{f \in \text{Lip}_1, \,t\in[0,T]}|\sum_{j=1}^k \lambda_jf(\phi_{t,x}^{\theta}(\om_{T,j}) - \CE[f(X_t^{\theta})]| \leq 3\,K \,C(m,\gamma)T^{1/2} k^{-\gamma/2}
\end{align*}
Then, by 1-Wasserstein duality we have:
\begin{align*}
    W_1(\pi_i, \hat{\pi_i}) = \sup_{f\in\text{Lip}_1}|\sum_{j=1}^k \lambda_j f(\phi_{t_i,x}^{\theta}(\om_{T,j}) - \CE[f(X_{t_i})]| \leq 3\,K \,C(m,\gamma)T^{1/2} k^{-\gamma/2}
\end{align*}
and thus, using \eqref{eq:wbd}:
\begin{align*}
    &|\CE(f(X_{t_1},\dots,X_{t_n}) - \sum_{j}\lambda_jf(\phi_{t_1,x}^{\theta}(\om_{T,j}),\dots,\phi_{t_n,x}^{\theta}(\om_{T,j}))| \leq 3\,K\,\lip \,C(m,\gamma)T^{1/2} k^{-\gamma/2}
\end{align*}
where, crucially, this bound is \textit{independent from n}. We exploit this in the final part, to follow.

\subsubsection*{Part 3}

Here we use Theorem~\ref{thm:stone-weierstrass} to complete the proof. Recall we denote $\Omega := C([0,T]; \mathbb{R}^d)$ the path space equipped with the uniform norm $\|\omega\|_\infty := \sup_{t \in [0,T]} \|\omega(t)\|_1$, and $C(\Omega)$ the space of continuous real-valued functionals on $\Omega$. Denote $\mathbb{Q}$ the infinite-dimensional path-space Law of $X_t^{\theta}$, and $\hat{\mathbb{Q}}$ the measure over the discrete set of paths $\phi_{t,x}^{\theta}(\om_{T,j})$, with weights $\lambda_j$. Then, we aim to show that 
\[\sup_{f\in \text{Lip}_1 \subset C(\Omega)}|\CE_{\mathbb{Q}}[f] - \CE_{\hat{\mathbb{Q}}}[f]| \leq 3\,K\,\lip \,C(m,\gamma)T^{1/2} k^{-\gamma/2}\]

We begin by identifying 
\[\sup_{f\in \text{Lip}_1}|\CE_{\mathbb{Q}}[f] - \CE_{\hat{\mathbb{Q}}}[f]| = W_1(\mathbb{Q},\hat{\mathbb{Q}})\]
A crucial observation is that by the uniform boundedness of vector fields (\ref{as:vecbd}), we have boundedness of paths $\phi_{t,x}^{\theta}(\om_{T,j})$. Also we have that the path-law $\mathbb{Q}$ is contained, in that for any $\delta >0$ we can find $M (\delta) < \infty$ such that $\sup_{t\in[0,T]}\mathbb{P}_{\mathbb{Q}}(\|X_t\| \leq M(\delta)) \geq 1-\delta$. Let us then choose some functional $\hat{f} \in \text{Lip}_1 \subset C(\Omega)$ such that $W_1(\mathbb{Q},\hat{\mathbb{Q}}) \leq |\CE_{\mathbb{Q}}[\hat{f}] - \CE_{\hat{\mathbb{Q}}}[\hat{f}]| + \epsilon/3$. Then, by Theorem~\ref{thm:stone-weierstrass}, for any $\epsilon > 0$ we may find some $n\in\nat$, and $F_n \in \CA$ such that 
\begin{equation}
\label{eq:epsbd}
|\CE_{\mathbb{P}}[F_n] - \CE_{\hat{\mathbb{P}}}[F_n]| \geq |\CE_{\mathbb{Q}}[\hat{f}] - \CE_{\hat{\mathbb{Q}}}[\hat{f}]| - \epsilon/2
\end{equation}
Here Theorem~\ref{thm:stone-weierstrass} is applied for a compact subset $K \subset C(\Omega)$, $K = \{f: \sup_{t\in[0,T]}\|f(t)\| \leq M(\delta)\}$ for some $\delta>0$ inducing $\sup_{f\in\text{Lip}_1}|(\CE_{\mathbb{Q|M(\delta)}}[f] - \CE_{\hat{\mathbb{Q}}|M(\delta)}[f]) - (\CE_{\mathbb{Q}}[f] - \CE_{\hat{\mathbb{Q}}}[f])| \leq \tilde{\epsilon}$, \\$\sup_{f\in\text{Lip}_1}|(\CE_{\mathbb{P|M(\delta)}}[f] - \CE_{\hat{\mathbb{P}}|M(\delta)}[f]) - (\CE_{\mathbb{P}}[f] - \CE_{\hat{\mathbb{P}}}[f])| \leq \tilde{\epsilon}$, where $\mathbb{Q}|M(\delta)$ denotes the measure $\mathbb{Q}$ restricted to $\|x\|\leq M(\delta)$ and renormalized. By this procedure we may make $\tilde{\epsilon}$ arbitrarily small and thus induce the bound \eqref{eq:epsbd} for any $\epsilon$. Thus, we have 
\begin{equation}
\label{eq:fnbd}
|\CE_{\mathbb{P}}[F_n] - \CE_{\hat{\mathbb{P}}}[F_n]| \geq W_1(\mathbb{Q},\hat{\mathbb{Q}}) -\epsilon
\end{equation}
from which it follows that 
\[\sup_{n\in\nat,F_n\in\CA}|\CE_{\mathbb{P}}[F_n] - \CE_{\hat{\mathbb{P}}}[F_n]| = W_1(\mathbb{P}, \hat{\mathbb{P}})\]
for clarity, we prove by contradiction: suppose otherwise. Then, there exists $\alpha> 0$ s.t. \\$\sup_{n\in\nat,F_n\in\CA}|\CE_{\mathbb{P}}[F_n] - \CE_{\hat{\mathbb{P}}}[F_n]| = W_1(\mathbb{P}, \hat{\mathbb{P}})-\alpha$, meaning \[W_1(\mathbb{P}, \hat{\mathbb{P}}) > \sup_{n\in\nat,F_n\in\CA}|\CE_{\mathbb{P}}[F_n] - \CE_{\hat{\mathbb{P}}}[F_n]| + \delta,\quad \forall \delta < \alpha\] which contradicts the fact that \eqref{eq:fnbd} can be achieved for any $\epsilon>0$.

Now, we simply write
\begin{align}
\begin{split}
    \sup_{f\in C(\Omega), f\in \text{Lip}_1}\left[\CE[f(X)] - \sum_{j=1}^k \lambda_k \phi(\om_{T,j})\right] &= W_1(\mathbb{Q},\hat{\mathbb{Q}}) = \sup_{n\in\nat,F_n\in\CA}|\CE_{\mathbb{P}}[F_n] - \CE_{\hat{\mathbb{P}}}[F_n]|\\
    & \leq 3\,K\,\lip\,C(m,\gamma)T^{1/2} k^{-\gamma/2}
\end{split}
\end{align}

\end{proof}

\subsubsection{Theorem~\ref{thm:bd2}}
\label{sec:pf4}
\begin{proof}
By Theorem~\ref{thm:cubap} and Theorem 19 of \cite{litterer2012high} we derive
\begin{align*}
    &\sup_{x,t} |P_t f(t,x) - \tPhi_{\lf}^{\theta}(t)| \\
    &\quad \leq \biggl(C_1\left( s_k^{1/2} + \sum_{i=1}^{k-1}\sum_{j=m}^{m+1} \frac{s_i^{(j+1)/2}}{(T-t_i)^{j/2}}\right)  + C_5\sum_{i=1}^{k-1} \frac{u_i^{r+1}}{(T-t_i)^{r\ps/2}}\biggr)\sup_{t\in[0,T]}\|\nabla f(t,\cdot))\|_{\infty}
\end{align*}
where $C_1:= 3\,T\,K \,\lfgsup$. By the results in \cite{lyons2004cubature} (also \cite{litterer2012high} page 20) we have, since $m-1 \geq \gamma$,
\[s_k^{1/2} + \sum_{i=1}^{k-1}\sum_{j=m}^{m+1} \frac{s_i^{(j+1)/2}}{(T-t_i)^{j/2}} \leq C_7(m,\gamma)T^{1/2} k^{-\gamma/2}\]
Let $\cfac$ be such that $\gamma \leq \frac{p(r+1)}{\frac{\ps r}{\cfac} + 1}$. Substituting $u_j = s_j^{\ps/2\gamma}$, we have 
\[\sum_{i=1}^{k-1} \frac{s_i^{\ps(r+1)/2\gamma}}{(T-t_i)^{r\ps/2}} \leq \sum_{i=1}^{k-1} c_i \frac{s_i^{(\frac{r\ps}{\cfac}+1)/2}}{(T-t_i)^{r\ps/2\cfac}} \leq C_6\sum_{i=1}^{k-1} \frac{s_i^{(\frac{r\ps}{\cfac}+1)/2}}{(T-t_i)^{r\ps/2\cfac}}\]
where 
\begin{align*}
    c_i = \frac{s_i^{(\frac{\ps(r+1)}{2\gamma} - \frac{\frac{\ps r}{\cfac}+1)}{2})}}{(T-t_i)^{r\ps/2c - r\ps/2}} &\leq T^{\ps(r+1) - \gamma(\frac{\ps r}{\cfac}+1)} k^{-\gamma(r\ps/2\cfac - r\ps/2)} \\
    & = T^{\ps(r+1) - \gamma(\frac{\ps r}{\cfac}+1)} k^{\gamma r\ps/2 - \gamma r\ps/2\cfac} =: C_6
\end{align*}

Then, since $r\ps-1\geq\gamma$,
\[ C_6\,\sum_{i=1}^{k-1} \frac{s_i^{(r\ps+1)/2}}{(T-t_i)^{r\ps/2}} \leq \,C_7(r\ps,\gamma) T^{1/2+\ps(r+1) - \gamma(\frac{\ps r}{\cfac}+1)}k^{\gamma r\ps/2 - \gamma r\ps/2\cfac-\gamma/2}\]
which gives us 
\begin{align}
\begin{split}
\label{eq:kbd}
    &\sup_{x,t} |P_t f(t,x) - \tilde{\Phi}_f^{\theta}(t)| \\
    & \quad \leq \left(C_1(m,\gamma,T) + C_2(r,\ps,\gamma,T) \right) \lfgsup T^{1/2}k^{\gamma r\ps/2 - \gamma r\ps/2\cfac-\gamma/2}
\end{split}
\end{align}
with $\lfgsup = \sup_{t\in[0,T]}\|\nabla f(t,\cdot)\|_{\infty}$ and $C_1$ and $C_2$ are constants depending only on their arguments. 

Now, to obtain \eqref{eq:cbound}, we first define 
\[\pmlen = \max_{i\in[r]}\text{length}(\om_i)\] 
as the maximum of the lengths of paths in the degree $m$ cubature formula on Wiener space over the unit time interval. Then, observe that given \ref{as:intfun} there exists $M'$ such that vector fields $\{V_i^{\theta}(\cdot)\}_{i=1}^{\bdim+1}$ are uniformly bounded by $M'$. Then (see \cite{litterer2012high} for more details), 
\[\text{supp}(Q_{\CD,u}^{(k)})(x) \subseteq B(x,\vfbd \pmlen \sum_{i=1}^{k} s_j^{1/2}) \subseteq B(x,\vfbd \pmlen kT^{1/2})\]
Thus, the total number of ODE computations is bounded by the total number of localizing balls of radius $u_j = s_j^{\ps/2\gamma}$ needed to cover the ball of radius $\vfbd \pmlen kT^{1/2}$.
The volume $V_{\xdim}(r)$ of an $\xdim$-dimensional ball of radius $r$ is given by
$V_{\xdim}(r) = \frac{\pi^{\xdim/2}}{\Gamma(\frac{\xdim}{2}+1)}r^{\xdim}$
where $\Gamma$ is the gamma function. Thus, the number of balls $\tilde{n}$ of radius $u_j = s_j^{\ps/2\gamma}$ necessary to cover $B(x,\vfbd \pmlen kT^{1/2})$ is given by 
\begin{equation}
\label{eq:nbound}
\tilde{n} = \frac{(\vfbd \pmlen kT^{1/2})^{\xdim}}{(s_j^{\ps/2\gamma})^{\xdim}} \leq \frac{(\vfbd\pmlen  kT^{1/2})^{\xdim}}{(T(\frac{1}{k})^{\gamma})^{\ps \xdim/2\gamma}} = (\vfbd \pmlen )^{\xdim} k^{\xdim(\ps/2 + 1)}T^{\xdim(\ps-1)}
\end{equation}
and the total number of ODE computations $n$ is then bounded by $(r+1)\tilde{n}$, since the reduced measure procedure (RMP) produces reduced measures with support cardinality of at most $r+1$ within each localizing ball. 
Inverting, we have \[k = \left(\frac{n}{r+1} (\vfbd \pmlen )^{-\xdim}T^{-\xdim(\ps-1)} \right)^{\frac{1}{\xdim(\ps/2+1)}} = (\vfbd \pmlen )^{-\frac{1}{(\ps/2-1)}}T^{-\frac{\ps-1}{\ps/2-1}}\left(\frac{n}{r+1}\right)^{\frac{1}{\xdim(\ps/2-1)}}\] Thus, substituting into \eqref{eq:kbd}, we obtain 
\begin{align}
\begin{split}
\label{eq:interbd}
    &\sup_{x,t} |P_t f(t,x) - \tilde{\Phi}_f^{\theta}(t)| \\
    &\quad \leq  (\vfbd \pmlen )^{\frac{\gamma}{\ps/2}}T^{\frac{\gamma(\ps-1)+1}{\ps-2}}\lfgsup \,T^{1/2}\left(C_1(m,\gamma,T) + C_2(r,\ps,\gamma,T)\right)\left(\frac{n}{r+1}\right)^{\frac{-\gamma\left(
    r\ps/\cfac - r\ps  + 1\right)}{\xdim(\ps-2)}}
\end{split}
\end{align}

Now, observe that the Proof of Theorem~\ref{thm:cubap} proceeds by obtaining a uniform pointwise bound such as \eqref{eq:interbd}, and produces a bound on the $W$-1 distance between pathwise SDE and cubature measures; in which the bound remains the same except for the functional Lipschitz constant $\lip$ replacing the Lipschitz constant $\lfgsup$. Thus, apply this methodology with the bound \eqref{eq:interbd} to produce \eqref{eq:cbound}, written more precisely as:
\begin{align}
\begin{split}
\label{eq:boundprec}
    &\left|\CE[\CLd(\Xt)] - \tPhi_{\CLd}^{\theta}\right| \\
    &\quad \leq (\vfbd \pmlen )^{\frac{\gamma}{\ps/2}}T^{\frac{\gamma(\ps-1)+1}{\ps-2}} \lip  \,T^{1/2}\left(C_1(m,\gamma,T) + C_2(r,\ps,\gamma,T)\right)\left(\frac{n}{r+1}\right)^{\frac{-\gamma\left(
    r\ps/\cfac - r\ps  + 1\right)}{\xdim(\ps-2)}}
\end{split}
\end{align}

\end{proof}

\bibliography{Bibliography.bib}

\end{document}